\documentclass{article}

\PassOptionsToPackage{numbers, compress}{natbib}

\usepackage[final]{neurips_2021}
\usepackage{notation}
\usepackage{algorithm,algpseudocode}

\usepackage{todonotes}
\usepackage{subcaption}
\usepackage{enumitem}

\usepackage[utf8]{inputenc} 
\usepackage[T1]{fontenc}    
\usepackage{hyperref}       
\usepackage{cleveref}
\usepackage{url}            
\usepackage{booktabs}       

\usepackage{nicefrac}       
\usepackage{microtype}      
\usepackage{xcolor}         

\usepackage{graphicx}
\usepackage{dsfont}
\usepackage{footnote}

\let\min\undefined
\DeclareMathOperator*{\min}{\vphantom{p}min} 
\let\max\undefined
\DeclareMathOperator*{\max}{\vphantom{p}max} 
\let\inf\undefined
\DeclareMathOperator*{\inf}{\vphantom{p}inf} 

\DeclareMathOperator*{\ex}{\mathbb{E}}

\usepackage{thm-restate}

\title{A/B/n Testing with Control in the Presence of Subpopulations}

\author{
  Yoan Russac
  \\
CNRS, Inria, ENS \\
Université PSL\\
  \texttt{yoan.russac@ens.fr} \\
   \And
   Christina Katsimerou \\
  Booking.com \\
   \texttt{christina.katsimerou@booking.com} \\
   \And
  Dennis Bohle\\
  Booking.com \\
  \texttt{dennis.bohle@booking.com} \\
   \And
   Olivier Cappé \\
  CNRS, Inria, ENS \\
  Université PSL\\
  \texttt{olivier.cappe@cnrs.fr} \\
   \And
   Aurélien Garivier \\
  UMPA, CNRS \\
  Inria, ENS Lyon \\
  \texttt{aurelien.garivier@ens-lyon.fr} \\
  \And
  Wouter M. Koolen \\
  Centrum Wiskunde \& Informatica
  \\
  \texttt{wmkoolen@cwi.nl}
}

\begin{document}
\normalsize

\maketitle

\begin{abstract}
  Motivated by A/B/n testing applications, we consider a finite set of
  distributions (called \emph{arms}), one of which is treated as a
  \emph{control}. We assume that the population is stratified into
  homogeneous subpopulations. At every time step, a subpopulation is
  sampled and an arm is chosen: the resulting observation is an
  independent draw from the arm conditioned on the subpopulation. The
  quality of each arm is assessed through a weighted combination of
  its subpopulation means. We propose a strategy for sequentially
  choosing one arm per time step so as to discover as fast as possible
  which arms, if any, have higher weighted expectation than the
  control. This strategy is shown to be asymptotically optimal in the
  following sense: if $\tau_\delta$ is the first time when the
  strategy ensures that it is able to output the correct answer with
  probability at least $1-\delta$, then $\mathbb{E}[\tau_\delta]$
  grows linearly with $\log(1/\delta)$ at the exact optimal rate. This
  rate is identified in the paper in three different settings: (1)
  when the experimenter does not observe the subpopulation
  information, (2) when the subpopulation of each sample is observed
  but not chosen, and (3) when the experimenter can select the
  subpopulation from which each response is sampled. We illustrate the
  efficiency of the proposed strategy with numerical simulations on
  synthetic and real data collected from an A/B/n experiment.
\end{abstract}

\section{Introduction}
A/B/n testing is a website optimization procedure where multiple
versions of the content (called "arms" below) are compared, often in order
to find the one with the highest conversion rate. However, many
e-commerce companies use A/B/n testing not only to deploy the best
product implementation, but primarily to draw post-experiment
inferences \cite{johari2015always}. The decision-making involves,
besides experiment results, factors such as the cost of scaling-up a
solution, external data, or whether the implementation fits in a
broader theme. In this setting, each of the arms better than the
default product (which we will refer to as the "control" arm) is a
contender for being deployed and the interest is not only in the best arm.

Given the control and $K \geq 1$ alternative implementations
(variants), the simplest idea is to distribute the traffic uniformly
among the arms; the arms that appear to be significantly better than
the control at the end of the experiment are considered for
deployment.
While well-established, this process can be inefficient in terms of
resources. Some alternatives are soon obviously worse (or better) than
the control and would require fewer samples than the alternatives
closer to the control. A second related shortcoming of the basic A/B/n
testing approach is that setting the duration of the experiment --when
done in advance-- necessitates a very conservative approach by
choosing a run-length that is sufficiently long to differentiate even the
smallest possible changes.

To address these limitations, we consider in this work sequential
testing policies that can both adjust the allocation of the samples
and be stopped adaptively, in light of the data gathered during the
experiment. In the terminology of multi-armed bandits, this
corresponds to \emph{pure exploration} problems (see, e.g., Chap. 33
of \cite{latsze20bandits}).
A pure exploration strategy will typically choose every minute (say),
an allocation of traffic that favors arms for
which the uncertainty is the highest. The experiment is stopped as
soon as the significance is considered sufficient for every arm.
Approaches have been developed
in~\cite{even2006action,kalyanakrishnan2012pac,garivier2016optimal} for the identification
of the single arm with
the highest mean, a task called the Best Arm Identification (BAI)
problem. In particular, \cite{garivier2016optimal} propose a strategy
that is asymptotically optimal in the \emph{fixed confidence setting}, meaning that,
given a risk parameter $\delta$, it finds the best arm with probability at least $1-\delta$, using
an expected number of samples that is hardly improvable when
$\delta$ is small. Later, \cite{yang2017framework} incorporated the
special role of the control arm in BAI and proposed an algorithm that
declares as winning arm the one with the highest mean only if it is
significantly better than the
control. 
In this paper, we propose a solution to the problem of identifying all
the arms that are better than the control, in a framework that generalizes the
fixed confidence
setting.
In order to provide useful tools for practical A/B/n testing, we
address two additional issues.

First, traditional stochastic bandit models are based on the
assumption that the arm samples are i.i.d., whereas real world data
streams usually show trends or some form of inhomogeneity. A particular
case of interest for website optimization are the seasonal patterns
caused by time-of-day or day-of-week variations. We henceforth include
in our model observed
covariates (e.g.\ the time of the day, but possibly also the country of
origin, or controlled covariates  like the
order in which partners appear on the page, etc.) that stratify the observations into homogeneous
subpopulations.
We study
different scenarios, depending on how much interaction is possible
with these subpopulations. We provide a sample complexity analysis and
an efficient algorithm in each case. In particular, we will show that
using the subpopulation information efficiently can provide significant
speedups of the decision-making. In the following,
we will refer to the task of identifying the set
of Arms that are Better than the Control in the presence of Subpopulations
as the ABC-S problem.

Second, the practice of A/B/n testing often differs from a pure
sequential experiment in that the experimenter cannot always
fix a risk $\delta$ at the beginning and passively wait for the
stopping time of the experiment without any time limitation.  To
address this issue, \cite{johari2015always} proposed to define some
notion of sequential "p-values" that can be monitored as the
experiment progresses and used to terminate it.
This notion was further used in the BAI setting in \cite{yang2017framework}.
In this contribution, we elaborate on this idea by sequentially updating a suggested
solution to the ABC-S problem \emph{together with} a
risk assessment for this suggestion. We show that, for any stopping time, the
probability that the suggested solution is incorrect is indeed lower than
the risk assessment. When the stopping time is selected as in usual fixed-confidence pure
exploration, we recover the exact same guarantees but this view of the problem
also provides useful results, for instance, if the
experiment needs to be terminated prematurely.

\paragraph{Related work.}
Pure exploration strategies have been studied in various settings:
the identification of the best arm \cite{even2006action, garivier2016optimal}, the
identification of the top $m$ arms \cite{chen2017nearly,kalyanakrishnan2012pac,gabillon2012best}
identifying the arms that are better than a threshold \cite{locatelli2016optimal, cheshire2020influence},
or identifying all $\epsilon$-good arms \cite{mason2020finding}.
As far as we know, this paper is the first to consider the problem of identifying all
the arms better than a control.
It is also the first to consider subpopulations in pure exploration tasks.
While motivated by the example of online companies, we believe that
the proposed algorithms are relevant to other domains where randomized
controlled trials are used for learning. An example could be clinical
trials: one may wish to identify all the alternative treatments that
work better that some reference medical treatment. This would permit
to choose among them taking into account different characteristics
(some could be cheaper, using another molecule for avoiding allergy,
etc.).

Close to the notion of the control is the notion of threshold.
\citet{locatelli2016optimal} propose an algorithm for identifying all arms above a
given threshold. Their algorithm
samples according to the significance of a statistical test, and
shares some similarities with the present article in the Gaussian
case; however, the perspective is rather different: the authors
consider the \emph{fixed-budget setting}: the total number of samples
is fixed, and the goal is then to minimize the probability of
returning a wrong answer at the end.
Here, the index of the control arm is known but its probability distribution is not.

In our work, the quality of
the different arms is assessed with a weighted combination of its subpopulations means.
Minimizing the estimation error of a convex combination of means through adaptive sampling
was considered in \cite{carpentier2011finite} with the introduction of a \textit{stratified estimator}
that will naturally appear in our analysis.

The paper is organized as follows. In Section~\ref{sec:ABC-S}, we present
the mathematical model and study the
information-theoretic complexity of the problem, extending the lower
bound of~\cite{garivier2016optimal} to the ABC-S setting. We show how the
complexity of the problem depends on the degree of interaction that one has
with the subpopulations, introducing different modes of interaction to be
defined in Figure~\ref{fig:modes} below. We also consider in detail the
Gaussian case which gives rise to more interpretable results.
Section~\ref{sec:algs} describes how to implement the proposed
strategy, which involves the numerical resolution of non-trivial
optimization problems. Finally, we provide the results of numerical
experiments on synthetic and real data sets in Section~\ref{sec:experiments}.

\section{The complexity of the ABC-S problem}
\label{sec:ABC-S}

\subsection{Mathematical framework}

A problem instance consists of the following ingredients. Known to the
learner are the number of arms $K \ge 1$ in addition to the designated
control arm $0$, the number of subpopulations $J$ (a standard bandit
being $J=1$), and the vector $\vbeta \in \mathbb R^J$
representing the relative importance
of the subpopulations for the learning objective.
We further make the stochastic
assumption that samples from each arm $a$ (including the control) and subpopulation $i$ are
drawn i.i.d.\ from an unknown probability distribution $\nu_{a,i}$ on
$\mathbb R$, whose mean we will denote by $\mu_{a,i}$. The quality of
arm $a$ is $\mu_a \df \sum_{i=1}^J \beta_i \mu_{a,i}$ the combination of the
means of the arms in the different populations. For $\vbeta \in \mathbb{R}^J$
we define the ABC-S problem as the correct identification
of the set
\[
\mathcal{S}_{\vbeta}(\vmu) \df \setc[\Big]{a \in [K]}{ \sum_{i=1}^J \beta_i \mu_{a,i} > \sum_{i=1}^J \beta_i \mu_{0,i}}\;.
\]

At every time step $t$, the algorithm selects an arm $A_t$ based on
previous choices and outcomes and observes or selects (except when explicitly
specified) the population type $I_t$.  Upon the selection of the arm
$A_t$ a reward $X_t$ is obtained.  This defines a sigma-field
generated by the observations up to time $t$ denoted
$\mathcal{F}_t = \sigma(I_1, X_1, \dotsc, I_t, X_t)$.  The number of
times arm $a$ was selected for subpopulation $i$ at time $t$ is denoted
$N_{a,i}(t) \df \sum_{s=1}^t \mathds{1}(A_s = a, I_s = i)$ and the
number of draws of arm $a$,
$N_{a}(t) \df \sum_{s=1}^t \mathds{1}(A_s = a)$. We define the gap with the control
arm and arm $a$, $\Delta_a \df \mu_0 - \mu_a$.

\paragraph{Modes of interaction} We consider four modes of interaction
of the learner with the bandit, as specified in
Figure~\ref{fig:modes} below.
In any of the three passive modes of interaction (described in
\Cref{fig:proportional,fig:agnostic,fig:oblivious}), we assume that the
subpopulation $i$ represents a known
proportion $\alpha_i$ of the total
population, and hence that the sequence of subpopulations is drawn i.i.d.\ from the fixed and discrete
distribution $I_t \sim \valpha = (\alpha_1, \dotsc, \alpha_J)$ with
$\valpha \in \Sigma_J \df \setc{x \in [0,1]^J}{\sum_{i} x_i = 1}$ the
$J$-dimensional simplex. Here $\valpha$ is an exogenous parameter
and can differ from $\vbeta$ which is inherent to
the learning objective and is also assumed to be known.
Although it is most natural in many applications to consider that
$\vbeta = \valpha$ (it is even necessary in the oblivious mode to make
the estimation of the $\mu_a$'s feasible), Example~\ref{ex:betaneg}
below describes a concrete scenario in which $\vbeta$ has negative
components.

\begin{figure}[h]
  \setlist[enumerate]{leftmargin=*}
  \begin{subfigure}[b]{.24\textwidth}
    \begin{enumerate}
    \item Pick $A_t$ and $I_t$
    \item See $X_t \sim \nu_{A_t, I_t}$
    \end{enumerate}
    \caption{\emph{Active} mode}\label{mode:active}
  \end{subfigure}
  \begin{subfigure}[b]{.24\textwidth}
    \begin{enumerate}
    \item See $I_t \sim \valpha$
    \item Pick $A_t$
    \item See $X_t \sim \nu_{A_t, I_t}$
    \end{enumerate}
    \caption{\emph{Proportional} mode}\label{mode:proportional}
    \label{fig:proportional}
  \end{subfigure}
  \begin{subfigure}[b]{.24\textwidth}
    \begin{enumerate}
    \item Pick $A_t$
    \item See $I_t \sim \valpha$
    \item See $X_t \sim \nu_{A_t, I_t}$
    \end{enumerate}
    \caption{\emph{Agnostic} mode}\label{mode:agnostic}
    \label{fig:agnostic}
  \end{subfigure}
  \begin{subfigure}[b]{.24\textwidth}
    \begin{enumerate}
    \item Pick $A_t$
    \item Do \emph{not} see $I_t \sim \valpha$
    \item See $X_t \sim \nu_{A_t, I_t}$
    \end{enumerate}
    \caption{\emph{Oblivious} mode.}
    \label{fig:oblivious}
  \end{subfigure}
  \caption{Modes of Interaction between Learner and Bandit in each
    round. In Active mode the learner determines the subpopulation,
    while in the right three passive modes it is sampled from
    $\valpha$.}
      \label{fig:modes}
\end{figure}

The distributions $(\nu_{a,i})_{a,i}$ are assumed to belong
the same one-parameter exponential family,
$
\mathcal{P} \df \{(\nu_\theta)_\theta: d \nu_\theta/d \xi = \exp( \theta x - b(\theta)) \}
$,
with $\xi$ a reference measure on $\mathbb{R}$ and
$b: \Theta \subset \mathbb{R} \mapsto \mathbb{R}$. Every probability
distribution $\nu_\theta$ in $\mathcal{P}$ is entirely defined by its
mean $\dot{b}(\theta)$ \cite{cappe2013kullback}. We may hence identify any bandit instance with its matrix of means $\vmu \in \mathbb R^{(K+1)\times J}$
In addition, the
Kullback-Leibler divergence between two distributions $\nu_\theta$ and
$\nu_{\theta'} \in \mathcal{P}$ may be written in the following Bregman
form:
\[
d(\mu, \mu') = \textnormal{KL}(\nu_\theta, \nu_{\theta'}) = b(\theta') - b(\theta) - \dot{b}(\theta)
(\theta'-\theta) \;,
\]
where $\mu = \dot{b}(\theta)$ and $\mu' = \dot{b}(\theta')$ correspond to the means of the two
distributions $\nu_\theta$ and $\nu_{\theta'}$.
We also use the notation $\textnormal{kl}(p,q)$ to denote the KL
divergence of two Bernoulli distributions of parameter $p$ and $q$.

We define
$\mathcal{L} \df \{\vmu : \forall a \in [K] \cup \{0\}, \forall i \in [J], \nu_{a,i} \in
\mathcal{P} \, \textnormal{and} \, \mu_{0} \neq \mu_{a} \}$
the set of identifiable instances where no arm has the
same weighted mean as the control. 
At every time step, the policies we consider
output a risk assessment $\hat{\delta}_t$ together with a recommendation $\hat{\mathcal{S}}_{t}$. 
We focus on \textit{safely
calibrated} policies, that are defined as satisfying the following
property
\begin{equation}
\label{eq:p-value}
\forall \vmu \in \mathcal{L},\  \forall \delta \in (0,1),  \quad
\mathbb{P}_\vmu \left(
  \exists t \geq 1 :
  \hat{\mathcal{S}}_{t} \neq  \mathcal{S}_{\vbeta}(\vmu) \;
  \cap \; \hat{\delta}_t \leq \delta \right) \leq \delta \;.
\end{equation}
Finally, when fixing a level of risk $\delta$, we consider the
stopping time associated to the filtration $\mathcal{F}_t$,
$\tau_\delta = \inf \{ t \geq 0, \, \hat{\delta}_t \leq \delta \}$.
The objective is then to minimize the expected number of rounds
necessary to obtain a level of risk of at most $\delta$. Contrary to
usual $\delta$-PAC algorithms if stopped before $\tau_\delta$, the
strategy still provides guarantees on the output set following
Equation~\ref{eq:p-value}. In particular, safely calibrated policies have a
sampling rule that does not depend on any pre-specified $\delta$, and
as such they are $\delta$-PAC for any $\delta$.

\begin{example}[{\cite[Largest Profit Identification problem][p24]{mixmart}}]
  \label{ex:betaneg}
  Consider a company choosing among $K$ product designs the model to mass produce. Each candidate design $k$ has an (equilibrium) sales price $\mu_{k,1}$ and production cost $\mu_{k,2}$. The goal is to find the model $k$ with the largest profit $\mu_{k,1} - \mu_{k,2}$. Prices and costs are currently unknown, but can be adaptively sampled. Sampling the "price" subpopulation $i=1$ is typically implemented by performing user preference studies, taking questionnaires, etc. Samples from the "cost" subpopulation $i=2$ involve rating manufacturing facilities, forecasting material and labor costs etc. This problem is interesting both in the BAI and ABC objectives. The importance vector is here $\vbeta = (1,-1)$ and 
  $\valpha$ has to be set by the learner.
\end{example}

\subsection{General form of the sample complexity}
Depending on the mode of interaction from Figure~\ref{fig:modes},
the learner has a set of sampling constraints to satisfy, here denoted
$\mathcal{C}$ and precisely defined in the next section.
We define $\Alt(\vmu)$, the different
problem instances where the set of arms better than the control
differs from that of the instance $\vmu$.
Formally,
$\Alt_{\vbeta}(\vmu) \df \setc{ \vlambda \in \mathcal{L}}{
  \mathcal{S}_{\vbeta}(\vlambda) \neq \mathcal{S}_{\vbeta}(\vmu)} $.
This allows us to bound the sample complexity.
\begin{restatable}{theorem}{samplecompl}
\label{th:generic_sample_comp}
Let $\delta \in (0,1)$ and $\vbeta \in \mathbb{R}^J$.  For any strategy
satisfying Equation~\ref{eq:p-value} and any
$\vmu \in\mathcal{L}$, the expected number of rounds for the ABC-S
problem for the agnostic, proportional and active mode satisfies:
\begin{equation}
\mathbb{E}_\vmu[\tau_{\delta}] \geq T^\star(\vmu) \,\textnormal{kl}(\delta, 1- \delta)
\quad
\textnormal{and}
\quad
\liminf_{\delta \to 0} \frac{\mathbb{E}_\vmu[\tau_\delta]}{\ln(1/\delta)}
\geq T^\star(\vmu) \;.
\label{eq:th_samplecompl}
\end{equation}
where (recalling that $\lambda_a = \sum_{i=1}^J \beta_i \lambda_{a,i}$)
\begin{align}
\label{eq:def_T_star_agnostic}
T^\star(\vmu)^{-1} &= \sup_{\w \in \mathcal{C}}
\inf_{\vlambda \in \textnormal{\Alt}_{\vbeta}(\vmu)}
\sum_{a=0}^K \sum_{i=1}^J w_{a,i} d(\mu_{a,i}, \lambda_{a,i})
  \\
\label{eq:seerank}
&=
\sup_{\w \in \mathcal{C}}
  \min_{b \neq 0}
  \inf_{
    \substack{\vlambda \in \mathcal L:
    \lambda_0 = \lambda_b}
  }
\sum_{a \in \{0,b\}} \sum_{i=1}^J w_{a,i} d(\mu_{a,i}, \lambda_{a,i})
 \;.
\end{align}
\end{restatable}
This result is established in
Appendix~\ref{app:lowerbd}. $T^\star$ characterizes the
difficulty of the learning problem.

\subsection{Influence of the mode of interaction}\label{sec:mode.as.constraints}
We consider the four different modes governing the sampling rule as
outlined in Figure~\ref{fig:modes}.
In the \textit{agnostic}
mode (Fig.~\ref{mode:agnostic}) an arm is first selected, after which the subpopulation type is observed.
Mathematically, this brings the
equality
$\mathbb{E}_\vmu[N_{a,i}(T) ] = \alpha_i \mathbb{E}_\vmu[N_a(T)]$
established in Lemma~\ref{lemma:agnostic} and the independence constraint on the
weights
$\w \in \mathcal{C}_{\textnormal{agnostic}} \df \{w_{a,i} = \alpha_i u_a
: (u_0, \dotsc, u_K) \in \Sigma_{K+1}\}$.

In the \textit{proportional} mode (Fig.~\ref{mode:proportional}), $A_t$ is
chosen based on $\mathcal{F}_{t-1}$ and the current
subpopulation
$I_t$. Here, the constraint is that the total number of pulls of the
different arms in the subpopulation $i$ should respect the
frequency of this subpopulation, i.e.\
$\sum_{a} \mathbb{E}_\vmu[N_{a,i}(T) ] = \alpha_i T$. This
induces a marginal constraint on the weights of the form
$ \w \in \mathcal{C}_{\textnormal{prop}} \df \setc{ \w \in
\Sigma_{(K+1)J}}{\forall i \leq J, \; \sum_a w_{a,i} = \alpha_i }$.
This result is established in Lemma~\ref{lemma:proportional} reported in
Appendix~\ref{app:interactions}.

In the \textit{active} mode (Fig.~\ref{mode:active}), the learner has an
additional degree of freedom--- she can ask for any subpopulation type at any round.
In that case,
$\w \in \mathcal{C}_{\textnormal{active}} \df \Sigma_{(K+1)J}$ is unconstrained.

By remarking that
$\mathcal{C}_{\textnormal{agnostic}} \subset
\mathcal{C}_{\textnormal{prop}} \subset
\mathcal{C}_{\textnormal{active}}$, and given the optimization program
\eqref{eq:def_T_star_agnostic} solved to obtain the characteristic
time, one immediately gets
\begin{equation}
\label{eq:ordering_times}
\forall \vmu \in \mathcal{L}, \quad
T_{\textnormal{active}}^\star(\vmu) \leq T_{\textnormal{proportional}}^\star(\vmu)
\leq T_{\textnormal{agnostic}} ^\star(\vmu) \;.
\end{equation}
Hence, as expected, the more control/information on the subpopulation
the learner has, the faster she is able to identify the set of arms
that are better than the control.

To compare with the \textit{oblivious} mode, in which the subpopulation
information is not even observed, we have to assume that $\valpha=\vbeta$.
In that case, the arm rewards follow a mixture distribution: $X_t | A_t \mathop{=} a ~\sim~ \sum_{i=1}^J \alpha_i \nu_{a,i}$.
In Proposition~\ref{prop:sample_comp_oblivious} reported in Appendix~\ref{app:oblivious}, we properly define
the characteristic time of an oblivious safely calibrated policy and prove that the joint convexity of Kullback-Leibler divergences implies
that it is larger than its agnostic counterpart.
This completes the picture of the ordering of the characteristic times by showing that, when $\valpha=\vbeta$,
\begin{equation}
\label{eq:ordering_times_complete}
\forall \vmu \in \mathcal{L}, \quad
T_{\textnormal{active}}^\star(\vmu) \leq T_{\textnormal{proportional}}^\star(\vmu)
\leq T_{\textnormal{agnostic}} ^\star(\vmu)
\leq T_{\textnormal{oblivious}} ^\star(\vmu) \;.
\end{equation}

Note that although we provide, in Section~\ref{sec:algs}, algorithms to
numerically compute the first three complexites, evaluating
$T_{\textnormal{oblivious}}^\star(\vmu)$ would be much harder, as the mixture distributions
can no more be parameterized by their mean only. Our current techniques do not yield a general-purpose practical
algorithm that is asymptotically optimal in the \textit{oblivious} mode for the ABC-S problem.
In the Bernoulli case, however, as mixtures of Bernoulli distributions are Bernoulli
distribution, one can use the single-population Bernoulli approach discussed in the next paragraph. For Gaussian
distributions, one can use a suboptimal approach based on the observation that location mixtures of Gaussians with bounded
means are sub-Gaussian (see Appendix~\ref{app:oblivious} for details).

\subsection{Single population and relationship with best arm identification}
In order to illustrate the nature of the the ABC-S
problem, we make a detour through the single population case, that is, when $J=1$.
Given two weights $w_a, w_b$ and two means $\mu_a$, $\mu_b$, we
introduce the minimum weighted transportation cost for moving the
means to a common position.
\[
  d_\text{mid}(w_a, \mu_a, w_b, \mu_b)
  ~\df~~
  \inf_v w_a d(\mu_a, v) + w_b d(\mu_b, v)
  = w_a d(\mu_a, v_{a,b}^\star) + w_b d(\mu_b, v_{a,b}^\star)
\]
where $v^*_{a,b}$, the optimal common location, is the weighted average,
i.e.
$
v^*_{a,b} = \frac{w_a}{w_a + w_b} \mu_a + \frac{w_b}{w_a + w_b}  \mu_b
$. 

\paragraph{Constructing an instance in the alternative}
When identifying all the arms better than a control,
there are two different ways to obtain a close-by bandit model $\vlambda$
in the alternative. The first option consists in taking an arm
which does not belong to $\mathcal{S}_{\vbeta}(\vmu)$ and to augment
its mean on the alternative
model such that it becomes above the control (or to reduce the mean of the control).
Otherwise, it is possible to take an arm that is better
than the control in the bandit model
$\vmu$ and to shrink its mean such that it becomes
lower than the control on the alternative (or augment the control).
Note that the infimum over the alternative has the same expression in the two cases
(see proof of Proposition~\ref{prop:charsingle} in Appendix~\ref{appx:charsingle}).

There is a priori no link between a BAI problem and an ABC one. In
particular, in the BAI problem there are only $K+1$ possible choices
for the best arm while when looking for $\mathcal{S}_{\vbeta}(\vmu)$ there are up
to $2^K$ different sets to consider. Yet, the next proposition shows that the
characteristic time $T^\star$ of any ABC problem with $J=1$ subpopulation shares strong similarities with
that of BAI problems.
\begin{restatable}{proposition}{charsingle}
\label{prop:charsingle}
Let $\delta \in (0,1)$ and $\vmu \in \mathcal{L}$. For any strategy satisfying
Equation~\ref{eq:p-value}, Equation~\ref{eq:th_samplecompl} holds with
\begin{align*}
\label{eq:def_T_star}
T^\star(\vmu)^{-1} & = \sup_{\w \in \Sigma_{K+1} } \inf_{\vlambda \in \Alt_{\vbeta}(\vmu)}
\sum_{a=0}^K w_a d(\mu_a, \lambda_a)
=
 \sup_{\w \in \Sigma_{K+1} }
\min_{b \neq 0} d_\textnormal{mid}(w_0, \mu_0, w_b, \mu_b) \;.
\end{align*}
\end{restatable}
The proof is reported in Appendix~\ref{sec:proof_char_single}.
Note that the expression of the sample complexity is really close to
the one in the BAI setting (\citet[Lemma 3]{garivier2016optimal})
except that we consider all the indices different from the control
here instead of the indices different from the best arm.  

\subsection{The Gaussian case}
\label{sec:gaussian_case}
In this section, we consider the Gaussian case which is of interest as the
characteristic time admits a more explicit expression, making it possible to further
investigate the differences between the various modes of interaction. We will state our results for the heteroscedastic case, in particular to get a closed-form proxy for the Bernoulli case, where each variance is a function of the (unknown) mean.

\paragraph{A/B testing}
When $K=1$ (one arm and the control arm), we are considering a standard A/B test
with subpopulations and one can easily prove the following result (established in Appendix~\ref{app:gaussian}).

\begin{restatable}{proposition}{lemmagaussfirst}
\label{proposition_others_weights}
For any
$\vmu \in \mathcal{L}$ with $K=1$ and $\nu_{a,i} = \mathcal{N}(\mu_{a,j}, \sigma_{a,j}^2)$ one has
\begin{enumerate}
\item $T_{\textnormal{agnostic}}^\star(\vmu) = \frac{2\left( \sqrt{\sum_{i=1}^J \frac{\beta_i^2 \sigma_{0,i}^2 }{\alpha_i}}
+
\sqrt{\sum_{i=1}^J \frac{\beta_i^2 \sigma_{1,i}^2 }{\alpha_i}}
\right)^2}{\Delta_1^2}  \;
 \;
 \textnormal{and}
 \;
w_{a,i}^\star =  \frac{ \alpha_i \sqrt{\sum_{i=1}^J \frac{\beta_i^2 \sigma_{a,i}^2 }{\alpha_i}}}{\sqrt{\sum_{i=1}^J \frac{\beta_i^2 \sigma_{0,i}^2 }{\alpha_i}} + \sqrt{\sum_{i=1}^J \frac{\beta_i^2 \sigma_{1,i}^2 }{\alpha_i}}} $
\item
$T_{\textnormal{prop}}^\star(\vmu) =  \frac{
2 \sum_{i=1}^J \frac{\beta_i^2}{\alpha_i} (\sigma_{0,i} + \sigma_{1,i})^2}
{\Delta_1^2}
\; \textnormal{and} \;
\forall i \leq J, \forall a \in \{0,1\}, \; w_{a,i}^\star = \frac{\alpha_i  \sigma_{a,i}}{\sigma_{0,i} + \sigma_{1,i}}
$
\item
$T_{\textnormal{active}}^\star(\vmu) = \frac{2 \left(
\sum_{i=1}^J |\beta_i|  (\sigma_{0,i} + \sigma_{1,i})
\right)^2}{\Delta_1^2} \;
\textnormal{and}
\;
\forall i \leq J,
\forall a \in \{0,1\}, \;
w_{a,i}^\star = \frac{ |\beta_i| \sigma_{a,i} }{\sum_{i=1}^{J} |\beta_i|
(\sigma_{0,i} + \sigma_{1,i})}
$
\end{enumerate}
\end{restatable}

The optimal allocations in the \textit{agnostic} and
\textit{proportional} cases are constrained by the
proportion of the different subpopulations $\valpha$, whereas, for the \textit{active} mode,
the optimal weights only depend
on $\vbeta$. In general, the optimal weights also depend
on the subpopulation variances, as is well-known in
stratified sampling estimation. Note however, that when (a) the subpopulations all
have a common variance $\sigma^2$ and (b) $\vbeta = \valpha$, then the
optimal allocations and the characteristic times are equal for
the \textit{agnostic}, the \textit{proportional} and the \textit{active} modes.
In that case, $w_{a,i}^\star = \alpha_i/2$, which
also corresponds to the well-known result in Gaussian A/B testing \cite{kaufmann2016complexity}. We have more generally
observed that whenever the subpopulations have approximately the same variances, the \textit{agnostic} and \textit{proportional} modes yield
very similar performances.

\paragraph{Weight computation in the homoscedastic case}
Even in scenarios where all subpopulation variances are equal to
$\sigma^2$, the \textit{active} mode remains very attractive in the cases where
$\vbeta \neq \valpha$. The following proposition shows that in that case, the optimal
weights for the ABC-S problem can be computed efficiently.

\begin{restatable}[Efficient computation in the Gaussian case]{proposition}{optimsimple}
\label{prop:optimal_weights_easy}
With Gaussian distributions with a known variance $\sigma^2$, letting
$ (u_0^\star, \dotsc, u_K^\star)
= \argmax_{u \in \Sigma_{K+1}} \min_{b \neq 0} \frac{\Delta_b^2}{2
 \left(\frac{1}{u_0} + \frac{1}{u_b} \right)  } $, the optimal weights for the active mode satisfy
\[
\forall a \in \{0, \dotsc, K \}, \; \forall i \leq J, \; w_{a,i}^\star = u_a^\star \frac{|\beta_i|}{\sum_{i=1}^J | \beta_i |}  \;.
\]
If, in addition $\valpha = \vbeta$, the above also holds for the agnostic and the proportional modes.
\end{restatable}
The interesting part of Proposition~\ref{prop:optimal_weights_easy} is that
computing $(u_0^\star, \dotsc, u_K^\star)$ can be done efficiently using
Theorem 5 from \cite{garivier2016optimal}. The optimal weights of
the ABC-S problem can be deduced from $u^\star$ without any further
calculation.

\section{Algorithms}\label{sec:algs}
To obtain our algorithms, we instantiate the Track-and-Stop algorithm template to our ABC-S problem. \citet{garivier2016optimal} introduced Track-and-Stop
 and proved its asymptotic optimality in the BAI setting.
Asymptotic optimality
for general partition identification problems was subsequently established
by \citet[Theorem~23]{mixmart} under the assumption of continuity of the oracle
weights $\vmu \mapsto \w^*(\vmu)$. \citet{multiple.answers} show that the
continuity assumption holds for all single-answer problems, in the upper-hemicontinuity
sense, which they show implies asymptotic optimality of the Track-and-Stop (T-a-S) algorithm.
These results directly apply to our ABC-S problem. \citet{purex.games} interpret
T-a-S as a noisy sequential equilibrium computation for the
max-min problem from the lower bound (e.g.\ Equation~\ref{eq:def_T_star_agnostic})
and develop computationally
attractive variants including lazy iterative solution of the $\w^*$ problem, and optimistic
gradients instead of forced exploration.

The details of our implementation are given in Appendix~\ref{sec:algdetails}. In short, we use the
simple standard $\Theta(\sqrt{t})$ forced exploration rounds,
a mode/subpopulation
aware upgrade of the D-tracking scheme \cite{garivier2016optimal} (which is empirically superior to C-tracking)
and we approximately and incrementally compute the oracle weights using the \emph{AdaHedge vs Best Response} iterative saddle point solver from \cite{purex.games}.
We use one single learner, instead of one per possible answer, as
advocated in \cite[Section~4]{purex.games}.
Note that we are not affected by
the non-convergence of D-tracking from \cite[Appendix~E]{purex.games}, as
our problem has a unique $\w^*$ because it is strictly concave in $\w$
(see Appendix~\ref{appx:alg.works}).

\paragraph{The sampling rule}
The high level overview of the algorithm is as follows.
We are given the number of arms $K$ and subpopulations $J$,
the exponential family, the mode of interaction, the subpopulation importance
coefficients $\vbeta$ and, for passive modes, their natural frequencies $\valpha$. The
algorithm then proceeds in rounds $t=1,2,\ldots$ Each round $t$, it calculates
the empirical frequencies $\hat \vmu_t \in \mathbb R^{(K+1) \times J}$
given by $\hat \mu_{a,i}(t) = \frac{1}{N_{a,i}(t)}
\sum_{s=1}^t X_s \mathbf 1\set*{A_s=a, I_s=i}$. It then computes
 (a suitable approximation of) the maximiser (i.e.\ the oracle policy)
  $\w_t = \w^*(\hat \vmu_t) \in \Sigma_{(K+1)\times J}$ of problem
  \eqref{eq:th_samplecompl}. In the active mode, we ``D-track'' $\w_t$, i.e.\ we
  sample $(A_t,I_t) \in \argmax_{a,i} N_{a,i}(t-1) - t \w_t(a,i)$. In the proportional mode, the
  subpopulation $I_t$ is given and we ``D-track'' the conditional distribution of $\w_t$ on
  arms given the subpopulation, i.e.\ $A_t \in \argmax_a N_{a,I_t}(t-1) - t \alpha_{I_t} \w_t(a|I_t)$, where $\w_t(a,i) = \alpha_i \w_t(a|i)$.
  In the agnostic mode we ``D-track'' the marginal distribution of $\w_t$ on
  arms, i.e.\ $A_t \in \argmax_a N_a(t-1) - t \w_t(a)$. For each mode, this sampling
  strategy ensures that $N_{a,i}(t) \approx t \w_t(a,i) \approx t w^*_{a,i}(\vmu)$, thus
  driving down the reported level of confidence as quickly as possible given
  the lower bound from Theorem~\ref{th:generic_sample_comp}.

\paragraph{The recommendation}
Concluding each round, we recommend $\mathcal S_{\vbeta}(\hat \vmu_t)$ at confidence level $\hat \delta(t) = \min \setc*{\delta \in (0,1)}{
  \Lambda(t)
  \ge \beta(t, \delta)
}$ obtained by inverting the threshold $\beta(t, \delta)$ at the GLR statistic
\begin{equation}\label{eq:GLR}
  \Lambda(t)
~=~
  \min_{b \neq 0}
  \inf_{
    \substack{\vlambda \in \mathcal L:
    \lambda_0 = \lambda_b}
  }
  \sum_{a \in \{0,b\}} \sum_{i=1}^J N_{a,i}(t) d(\hat \mu_{a,i}(t), \lambda_{a,i})
  \;.
\end{equation}

\paragraph{The threshold}
For the sharpest theoretically supported
thresholds we refer to \cite{mixmart}.
Namely, an ABC-S problem with $K$-arms and $J$-subpopulations has $2^K$
 answers, and its \emph{rank} \cite[Definition~22]{mixmart}
  is $2 J$, as can be read off from \eqref{eq:seerank}.
  By \cite[Proposition~23]{mixmart} we have validity
   for $\beta(t, \delta) = 6 J \ln \ln t +
   \ln \frac{1}{\delta} + K + 2 J \cdot O(\ln \ln \frac{1}{\delta})$.
   In practice, we follow \cite{garivier2016optimal}
and use instead  the heavily stylized $\ln ((1+\ln t)/\delta)$ that omits several union bounds.

\begin{theorem}\label{thm:asym.opt}
  For every mode, Subpopulation Track-and-Stop is safely calibrated
  (Equation~\ref{eq:p-value}).
  Moreover, Subpopulation Track-and-Stop is asymptotically optimal and
   matches
  the lower bound from Theorem~\ref{th:generic_sample_comp}, in the sense that
\[
\textnormal{for every bandit} \; \vmu \in \mathcal L, \;
\lim_{\delta \to 0} \frac{\mathbb{E}[\tau_\delta]}{\ln(1/\delta)}
= T^\star(\vmu) \;.
\]
\end{theorem}

We include the proof in Appendix~\ref{appx:alg.works}.

\section{Experiments}
\label{sec:experiments}
\subsection{Simulations}
\label{sec:simulations}
We conduct numerical experiments to
evaluate the proposed algorithms, focusing on Bernoulli bandit
models, which are ubiquitous in practical applications.

In our experiments, in addition to our T-a-S algorithms with the various interaction
modes,
we include two more sampling rules for comparison:
(1) uniform sampling as a baseline, and (2) the experimentally efficient \textit{Best Challenger} heuristic inspired by \cite{garivier2016optimal},
adapted to the ABC problem and denoted BC-ABC in the sequel.
BC~\cite{garivier2016optimal} for the BAI problem samples in every round the empirical
best arm $\hat{a}_t$
or its best challenger, i.e.\ the arm $\hat{c}_t \neq \hat{a}_t$ at which the GLR statistic (Equation~\ref{eq:GLR}) reaches its minimum.
Our BC-ABC adaptation
samples in every round the control arm or the arm that yields the minimum GLR statistic $\Lambda(t)$,
in the agnostic interaction mode (since $\Lambda(t)$ is subpopulation independent). 
For clearer comparison
between the sampling strategies, all algorithms use the Chernoff stopping criterion~\cite{garivier2016optimal} to determine either when to stop or output the risk assessment at a given time.
We also opted for sampling rules independent from the confidence
parameter $\delta$, because we are aiming for safely calibrated policies.

We first illustrate the fact that the T-a-S algorithm provides a correct --but rather conservative-- assessment
of the risk of its decision whatever the time it is stopped at. To do so, we generated 1000 bandit instances
uniformly at random from [0, 1] with $K=2$ arms.
For each instance, we recorded the first time a certain risk assessment level is
reached and the correctness of the algorithm's recommendation at that point.
We map to each risk assessment level the proportion of errors across all instances. We chose
two stopping rates that are not supported by theory but are recommended in practice
\cite{garivier2016optimal}.
Figure~\ref{fig::boxplots} (Left)
illustrates the isotonic curve fitted on our observations and suggests that even the most
 lenient stopping threshold $\ln((\ln(t)+1)/ \delta) $ results in much lower empirical probability of error than the risk assessment. 
 In the following, we use the stopping threshold $\ln((\ln(t)+1)/ \delta) $.

 \begin{figure*}[hbt]
\begin{subfigure}{.49\textwidth}
  \centering
        \includegraphics[width=1\linewidth]{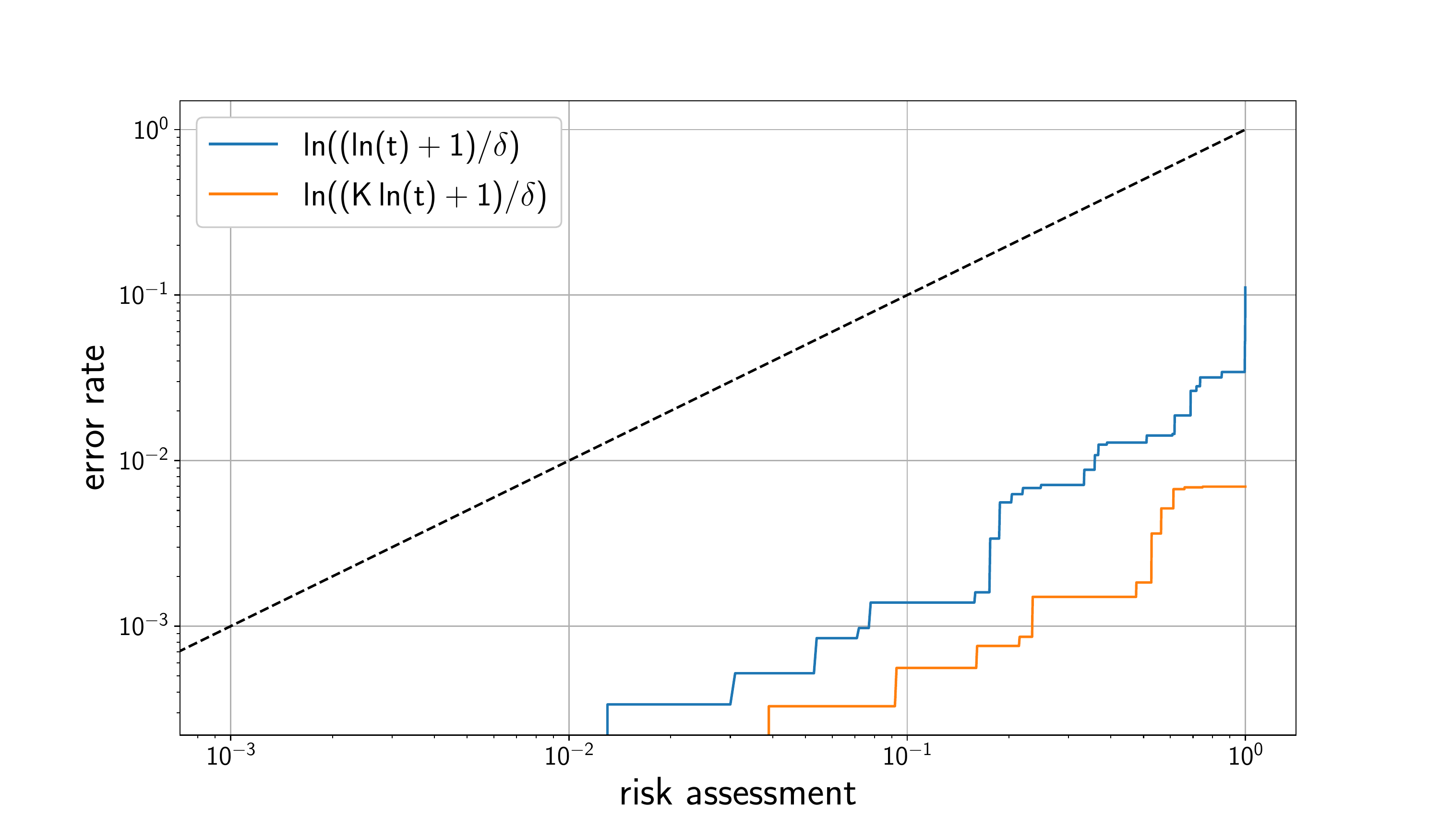}
        \label{fig::isotonic}
    \end{subfigure}
   \begin{subfigure}{.49\textwidth}
  \centering
  \includegraphics[width=1\linewidth]{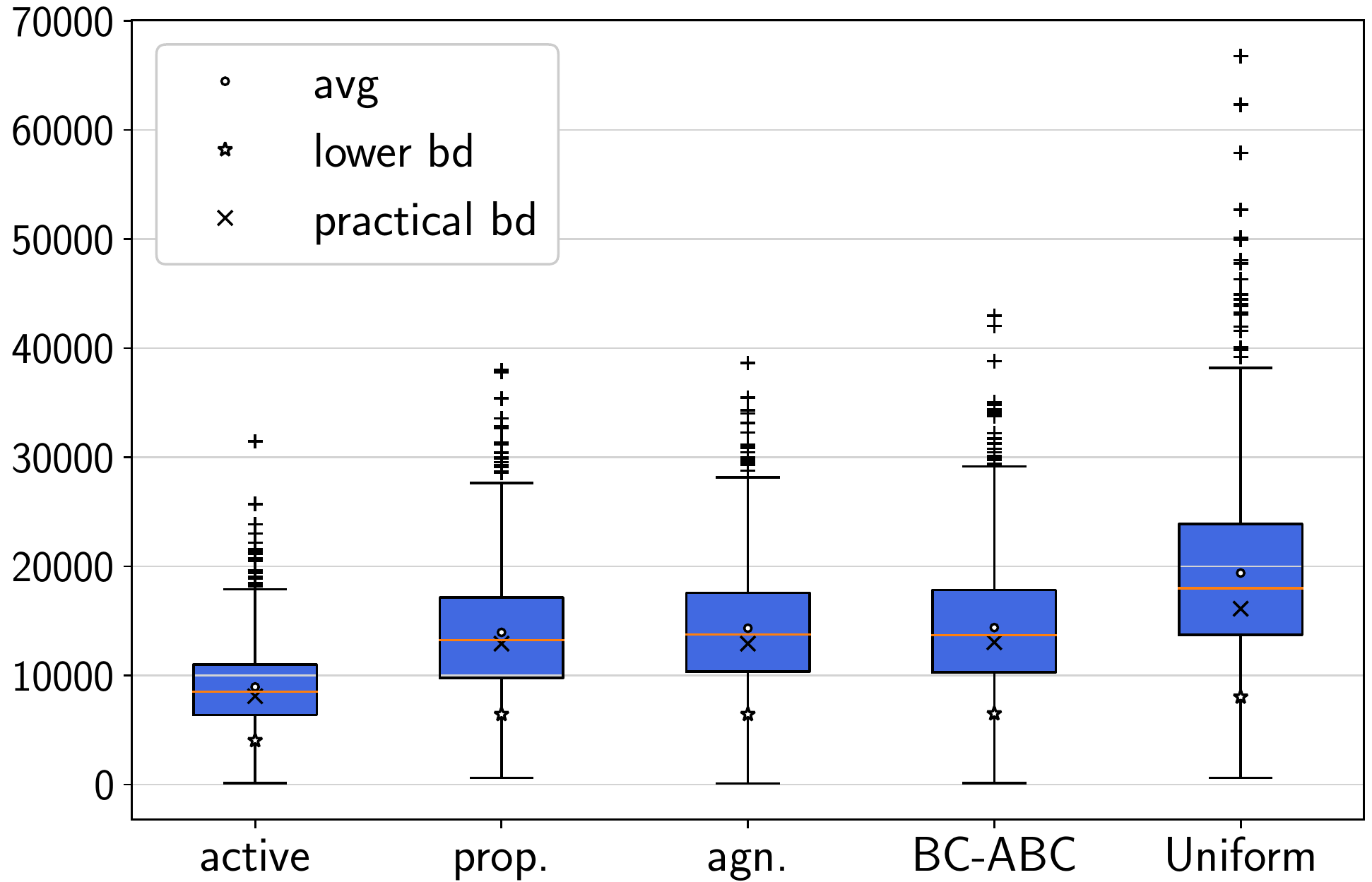}
        \label{fig::notequal_a_b}
     \end{subfigure}
\caption{(Left) Risk assessment calibration on a log-log scale. (Right) Stopping time boxplot for $\bm{\mu} = [0.1\: 0.4\: 0.3; 0.2\: 0.5\: 0.2; 0.5\: 0.1\: 0.1] \in [0,1]^{(K+1) \times J} $
when $\vbeta = [1/3,1/3,1/3], \valpha = [0.4, 0.5, 0.1]$ with Bernoulli distributions.}
\label{fig::boxplots}
\end{figure*}

In our second experiment \footnote{Code at \url{https://gitlab.com/ckatsimerou/abc_s_public}}, we generated 3000 Bernoulli bandit instances with $ K = 2 $ and a random number of
subpopulations $ J $ between 2 and 10.
Each subpopulation-arm’s mean $\mu_{a,i}$ is drawn uniformly at random from $[0, 1]$, and the subpopulation frequency vector $\valpha$
is drawn from a Dirichlet$(10)$ distribution.
Table~\ref{tab::sample complexity} reports the average stopping time of
each algorithm across all bandit instances. On average, the T-a-S algorithms at all modes stop at similar times,
and all adaptive sampling methods terminate faster than
uniform sampling.

\begin{table}[ht]
\caption{Average stopping time. Description in text.}
\centering
\begin{tabular}{c c c c c}
\hline\hline
T-a-S (active)  & T-a-S (proportional) & T-a-S (agnostic) & BC-ABC &  Uniform \\ [0.5ex]
\hline %
14871 & 15231 &  15444 & 15279 &  21586  \\
\hline
\end{tabular}
\label{tab::sample complexity}
\end{table}

To better understand the role of $\vbeta$ and $\valpha$, we ran the algorithms on a specific model (see Figure~\ref{fig::boxplots}, Right)
with $\valpha \neq \vbeta $. In this case, the optimal proportions are constrained by the frequencies
of the subpopulation for passive interaction modes.
The expected number of samples needed to identify the ABC-S solution is
lower for the active policy, which has an additional degree of freedom in its
sampling strategy. The \textit{proportional} interaction mode and the \textit{agnostic} interaction modes perform similarly.
As expected, all the proposed strategies outperform the uniform sampling rule.
We contrast the stopping time with the lower
bound $ \text{kl}(\delta, 1-\delta)T^*(\bm{\mu})$, and with a more practical version, which indicates, approximately, the first time at which the GLR statistic crosses the
threshold, i.e.\ solving $t = \ln((\ln(t)+1)/\delta)T^*(\bm{\mu})$, as was done
in \cite{purex.games}.
All adaptive algorithms perform well on this instance, with their average runtime being very close to their respective practical bound.

\subsection{Application to A/B/n experiment}
We evaluate the algorithms on data collected from an actual A/B/n
experiment, which compares different copies of a component of the
webpage, in order to identify the ones better than the default copy.
The metric of interest is whether the visitor clicked at least once
during the experiment to the next page after getting exposed to one of
the variants. For this setup we considered $K=2$ copies competing
against the control, with each copy being treated as an arm. Due to
global traffic, the data exhibits strong seasonality patterns within a
day, as seen in Figure~\ref {fig::seasonal_data}, in which every point corresponds to
click-through rate per six hours (quarter of day) for 12 consecutive
days. We treat the $J=4$ seasons as i.i.d.\ subpopulations. Within each season
we shuffled the data to eliminate the weekly trend.

The summary statistics of the dataset, together with the characteristic times
and the optimal weights for each T-a-S mode can be found in Appendix~\ref{appx:booking.data.lbs}.
Note that the small gaps between the arm means makes this practical ABC-S problem much harder than
the synthetically generated examples.

We tested all algorithms described in Section~\ref{sec:simulations}.
Each algorithm terminates when it reaches for the first time $\hat\delta_t \le 0.1$ or outputs a risk assessment on the
recommendation if it runs out of samples, which in this experiment occurs after $1.4 \cdot 10^7$ observations.
Here, we weigh the importance $\vbeta$ of each season equally to its observed frequency $\valpha$.
Doing so, we do not expect large performance discrepancies between the different T-a-S interaction modes, which is confirmed
by their characteristic times (Appendix~\ref{appx:booking.data.lbs}).
The observations from
Fig.~\ref{fig::risks}
are similar to the results from the numerical simulations:
 adaptive sampling achieves lower sample complexity over uniform sampling and T-a-S for the active interaction mode
 terminates faster than for the passive modes.
 All algorithms yield the correct recommendation, but not with the same risk assessment. All T-a-S algorithms terminated within
 the available sample size, BC-ABC almost terminated and output a risk assessment slightly above 0.1 and uniform's risk assessment was 0.67.
 Of course, when viewing seasonality as a subpopulation, the \textit{active} mode is unrealistic, but it is still informative to see that it can be very economical in hard problems in which sampling the subpopulations
 actively is an option.
 In this instance, \textit{proportional}, \textit{agnostic} and \textit{oblivious} modes terminated at similar times. However, we would recommend using the \textit{proportional}
 mode, given that we expect it to never perform worse than the other passive modes on average.
One should not be surprised by the curve for the uniform sampling, this policy was stopped before convergence because it ran out of samples.

Lastly, here we assumed that seasons occur in i.i.d.\ fashion, but in reality there is temporal dependence between them. This imposes
extra constraints on the optimal weights and increases the sample complexity.
However, we do not expect this to be detrimental for cases in which seasons alternate frequently
and full cycles are observed often, as was the case with our example.

\begin{figure*}[hbt]
	\begin{subfigure}{.49\textwidth}
  	\centering
 	\includegraphics[width=1\textwidth]{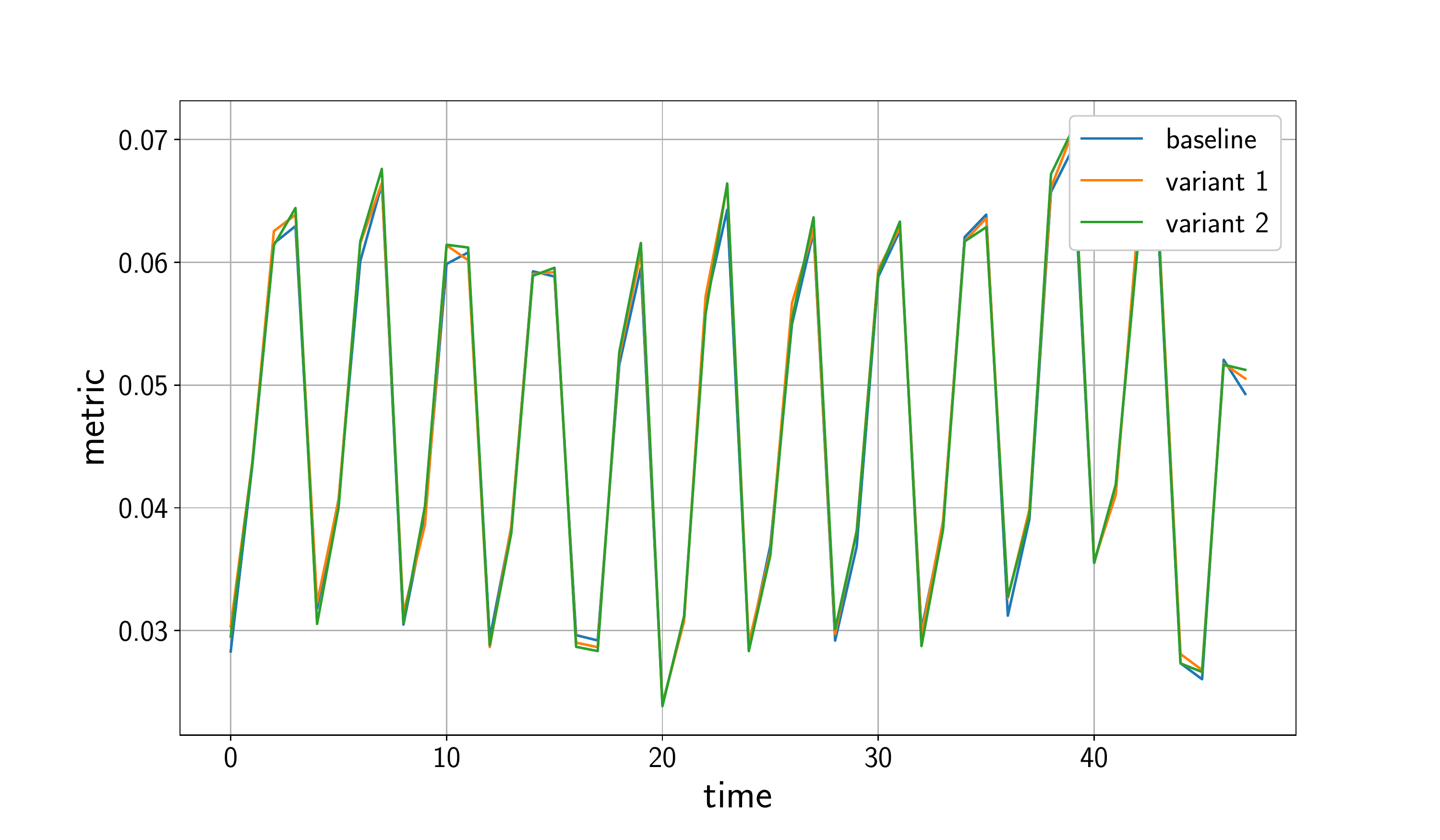}
        \caption{Click-through rate per 6 hours for 12 days.}
        \label{fig::seasonal_data}
    	\end{subfigure}
   	\begin{subfigure}{.49\textwidth}
 	 \centering
  	\includegraphics[width=1\textwidth]{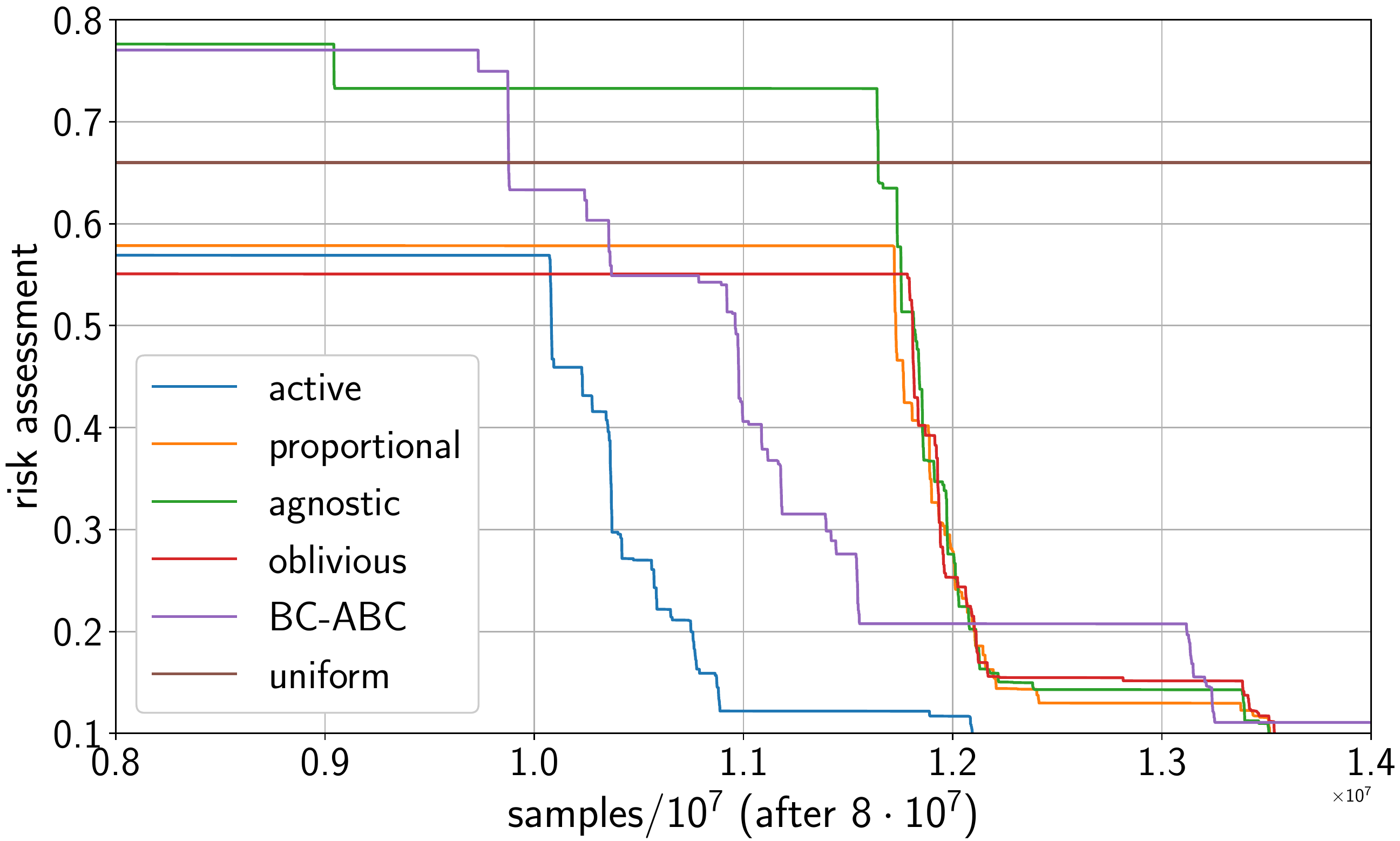}
          \caption{Risk assessment over time.}
        \label{fig::risks}
     	\end{subfigure}
\caption{Real data and results.}
\end{figure*}

\section{Conclusion}

In this work, we considered the pure exploration task of identifying all
the arms that are better than a control arm in the presence of
subpopulations (ABC-S). We design asymptotically optimal policies for
this problem under different assumptions on the mode of interaction between the learner and the bandit. We observed that the
\textit{active} mode, in which the learner decides which subpopulation it
samples, may significantly reduce decision times. On the other hand,
the other modes, in which the learner has to respect the natural
proportions of the different subpopulations (i.e., in
\textit{proportional} and \textit{agnostic} modes) produce more modest
effects, except when the subpopulations differ significantly in
variances. Finally, we proposed a natural way to provide anytime
decisions with risk guarantees in the Track-and-Stop framework.

\section{Potential Societal Impact}

The contributions presented in this work are mostly related to methods and, as such, do not have a direct expected societal impact. This being said, a potential concern that will need to be addressed more carefully in subsequent applications of these methods is the use of subpopulation information, which could be exploited to target specific user behaviour or characteristics. In the use case considered in Section 4.2, the subpopulations correspond to time slots that are used to model seasonality in the user responses, which does not raise any specific ethical concern. However, in cases where the subpopulations are formed using characteristics of individual users, the impact needs to be assessed more thoroughly. Note that in such cases, restricting to one of the more conservative modes of interaction (i.e.\ agnostic or even oblivious) may become necessary in order to prevent undue use of population-dependent information. \\

{\bf \large{Acknowledgment}}
\\
The authors would like to thank the anonymous reviewers whose comments and questions helped improve the clarity of this manuscript. A. Garivier acknowledges the support of the Project IDEXLYON of the University of Lyon, in the framework of the Programme Investissements d’Avenir (ANR-16-IDEX-0005), and Chaire SeqALO (ANR-20-CHIA-0020).

\clearpage
\newpage

\bibliographystyle{abbrvnat}
\bibliography{refs}

\begin{thebibliography}{18}
\providecommand{\natexlab}[1]{#1}
\providecommand{\url}[1]{\texttt{#1}}
\expandafter\ifx\csname urlstyle\endcsname\relax
  \providecommand{\doi}[1]{doi: #1}\else
  \providecommand{\doi}{doi: \begingroup \urlstyle{rm}\Url}\fi

\bibitem[Capp{\'e} et~al.(2013)Capp{\'e}, Garivier, Maillard, Munos, Stoltz,
  et~al.]{cappe2013kullback}
O.~Capp{\'e}, A.~Garivier, O.-A. Maillard, R.~Munos, G.~Stoltz, et~al.
\newblock Kullback--leibler upper confidence bounds for optimal sequential
  allocation.
\newblock \emph{Annals of Statistics}, 41\penalty0 (3):\penalty0 1516--1541,
  2013.

\bibitem[Carpentier and Munos(2011)]{carpentier2011finite}
A.~Carpentier and R.~Munos.
\newblock Finite-time analysis of stratified sampling for monte carlo.
\newblock In \emph{NIPS-Twenty-Fifth Annual Conference on Neural Information
  Processing Systems}, 2011.

\bibitem[Chen et~al.(2017)Chen, Li, and Qiao]{chen2017nearly}
L.~Chen, J.~Li, and M.~Qiao.
\newblock Nearly instance optimal sample complexity bounds for top-k arm
  selection.
\newblock In \emph{Artificial Intelligence and Statistics}, pages 101--110.
  PMLR, 2017.

\bibitem[Cheshire et~al.(2020)Cheshire, Menard, and
  Carpentier]{cheshire2020influence}
J.~Cheshire, P.~Menard, and A.~Carpentier.
\newblock The influence of shape constraints on the thresholding bandit
  problem.
\newblock In \emph{Conference on Learning Theory}, pages 1228--1275. PMLR,
  2020.

\bibitem[de~Rooij et~al.(2014)de~Rooij, van Erven, Gr{\"u}nwald, and
  Koolen]{ftl.jmlr}
S.~de~Rooij, T.~van Erven, P.~Gr{\"u}nwald, and W.~M. Koolen.
\newblock Follow the leader if you can, {H}edge if you must.
\newblock \emph{Journal of Machine Learning Research}, 15:\penalty0 1281--1316,
  Apr. 2014.

\bibitem[Degenne and Koolen(2019)]{multiple.answers}
R.~Degenne and W.~M. Koolen.
\newblock Pure exploration with multiple correct answers.
\newblock In \emph{Advances in Neural Information Processing Systems (NeurIPS)
  32}, pages 14591--14600. Dec. 2019.

\bibitem[Degenne et~al.(2019)Degenne, Koolen, and M\'enard]{purex.games}
R.~Degenne, W.~M. Koolen, and P.~M\'enard.
\newblock Non-asymptotic pure exploration by solving games.
\newblock In \emph{Advances in Neural Information Processing Systems (NeurIPS)
  32}, pages 14492--14501. Dec. 2019.

\bibitem[Even-Dar et~al.(2006)Even-Dar, Mannor, Mansour, and
  Mahadevan]{even2006action}
E.~Even-Dar, S.~Mannor, Y.~Mansour, and S.~Mahadevan.
\newblock Action elimination and stopping conditions for the multi-armed bandit
  and reinforcement learning problems.
\newblock \emph{Journal of machine learning research}, 7\penalty0 (6), 2006.

\bibitem[Gabillon et~al.(2012)Gabillon, Ghavamzadeh, and
  Lazaric]{gabillon2012best}
V.~Gabillon, M.~Ghavamzadeh, and A.~Lazaric.
\newblock Best arm identification: A unified approach to fixed budget and fixed
  confidence.
\newblock In \emph{NIPS-Twenty-Sixth Annual Conference on Neural Information
  Processing Systems}, 2012.

\bibitem[Garivier and Kaufmann(2016)]{garivier2016optimal}
A.~Garivier and E.~Kaufmann.
\newblock Optimal best arm identification with fixed confidence.
\newblock In \emph{Conference on Learning Theory}, pages 998--1027. PMLR, 2016.

\bibitem[Johari et~al.(2015)Johari, Pekelis, and Walsh]{johari2015always}
R.~Johari, L.~Pekelis, and D.~J. Walsh.
\newblock Always valid inference: Bringing sequential analysis to {A/B}
  testing.
\newblock \emph{arXiv preprint arXiv:1512.04922}, 2015.

\bibitem[Kalyanakrishnan et~al.(2012)Kalyanakrishnan, Tewari, Auer, and
  Stone]{kalyanakrishnan2012pac}
S.~Kalyanakrishnan, A.~Tewari, P.~Auer, and P.~Stone.
\newblock Pac subset selection in stochastic multi-armed bandits.
\newblock In \emph{ICML}, volume~12, pages 655--662, 2012.

\bibitem[Kaufmann and Koolen(2018)]{mixmart}
E.~Kaufmann and W.~M. Koolen.
\newblock Mixture martingales revisited with applications to sequential tests
  and confidence intervals.
\newblock Preprint, Oct. 2018.

\bibitem[Kaufmann et~al.(2016)Kaufmann, Capp{\'e}, and
  Garivier]{kaufmann2016complexity}
E.~Kaufmann, O.~Capp{\'e}, and A.~Garivier.
\newblock On the complexity of best-arm identification in multi-armed bandit
  models.
\newblock \emph{The Journal of Machine Learning Research}, 17\penalty0
  (1):\penalty0 1--42, 2016.

\bibitem[Lattimore and Szepesvári(2020)]{latsze20bandits}
T.~Lattimore and C.~Szepesvári.
\newblock \emph{Bandit Algorithms}.
\newblock Cambridge University Press, 2020.
\newblock \doi{10.1017/9781108571401}.

\bibitem[Locatelli et~al.(2016)Locatelli, Gutzeit, and
  Carpentier]{locatelli2016optimal}
A.~Locatelli, M.~Gutzeit, and A.~Carpentier.
\newblock An optimal algorithm for the thresholding bandit problem.
\newblock In \emph{International Conference on Machine Learning}, pages
  1690--1698. PMLR, 2016.

\bibitem[Mason et~al.(2020)Mason, Jain, Tripathy, and Nowak]{mason2020finding}
B.~Mason, L.~Jain, A.~Tripathy, and R.~Nowak.
\newblock Finding all $\{\epsilon\}$-good arms in stochastic bandits.
\newblock \emph{Advances in Neural Information Processing Systems}, 2020.

\bibitem[Yang et~al.(2017)Yang, Ramdas, Jamieson, and
  Wainwright]{yang2017framework}
F.~Yang, A.~Ramdas, K.~Jamieson, and M.~J. Wainwright.
\newblock A framework for {Multi-A(rmed)/B(andit)} testing with online {FDR}
  control.
\newblock In \emph{Advances in Neural Information Processing Systems}, 2017.

\end{thebibliography}

\clearpage
\newpage
\appendix

\section{General form of the characteristic time}
\label{app:lowerbd}
\subsection{Proof of Theorem~\ref{th:generic_sample_comp}}
\samplecompl*

\begin{proof}
  Using the transportation lemma from \cite{kaufmann2016complexity}
  and recalling that $N_{a,i}(t)$ is the number of draws of arm $a$ in
  subpopulation $i$ up to time $t$, we have for any safely calibrated
  policies
\[
\forall \vlambda \in \Alt_{\vbeta}(\vmu), \;
\sum_{i=1}^J \sum_{a=0}^K \mathbb{E}_\vmu[N_{a,i}(\tau_\delta)]
d(\mu_{a,i}, \lambda_{a,i}) \geq \textnormal{kl}(\delta, 1- \delta) \;.
\]
Therefore,
\begin{align*}
\textnormal{kl}(\delta, 1- \delta)
& \leq
\inf_{\vlambda \in \Alt_{\vbeta}(\vmu)}
\sum_{a=0}^K \sum_{i=1}^J \mathbb{E}_\vmu[N_{a,i}(\tau_\delta)]
d(\mu_{a,i}, \lambda_{a,i}) \\
& =
\mathbb{E}_\vmu[\tau_\delta]
\inf_{\vlambda \in \Alt_{\vbeta}(\vmu)}
\sum_{a=0}^K \sum_{i=1}^J
\frac{\mathbb{E}_\vmu[N_{a,i}(\tau_\delta)] }{\mathbb{E}_\vmu[\tau_\delta] }
d(\mu_{a,i}, \lambda_{a,i}) \\
& \leq
\mathbb{E}_\vmu[\tau_\delta]
\sup_{\w \in \mathcal{C}}
\inf_{\vlambda \in \Alt_{\vbeta}(\vmu)}
\sum_{a=0}^K \sum_{i=1}^J
w_{a,i}
d(\mu_{a,i}, \lambda_{a,i})\;.
\end{align*}
In the last inequality, we used the fact that the normalized expected numbers of draws
satisfy the set of constraints defined by $\mathcal{C} \subset \Sigma_{(K+1)J}$.
Using $\textnormal{kl}(\delta, 1-\delta) \sim \ln(1/\delta)$ when $\delta$
tends to 0 gives the first result.

We denote
$\Lambda(\w, \vlambda, \vmu) \df \sum_{a=0}^K \sum_{i=1}^J w_{a,i}
d(\mu_{a,i}, \lambda_{a,i})$. To obtain the second result, we will
simplify the expression of $T^\star(\vmu)^{-1}$.  Using that the KL
divergences and the weights are positive, for $\vlambda$ to be in the
alternative, one of the two following conditions need to be met: (1)
there exists $a \in \mathcal{S}_{\vbeta}(\vmu)$ such that
$\lambda_a < \lambda_0$. (2) there exists
$a \in \mathcal{S}_{\vbeta}^-(\vmu) \df \setc{a \in [K]}{\mu_a < \mu_0}$ such that
$\lambda_a > \lambda_0$.

For this reason, one has
\[
\inf_{\vlambda \in \Alt_{\vbeta}(\vmu)}
\Lambda(\w, \vlambda, \vmu)
= \min \left(
\min_{a \in \mathcal{S}_{\vbeta}(\vmu)}
\inf_{\vlambda: \lambda_a < \lambda_0} \Lambda(\w, \vlambda, \vmu),
\min_{a \in \mathcal{S}_{\vbeta}^-(\vmu)}
\inf_{\vlambda: \lambda_a > \lambda_0} \Lambda(\w, \vlambda, \vmu)
\right) \;.
\]

We obtain the desired result by remarking that the inner optimization programs $\inf_\vlambda$ are each achieved on the boundary (the constraint being satisfied with equality) where they coincide,
and that $\{1, \dotsc, K\} = \mathcal{S}_{\vbeta}(\vmu) \cup \mathcal{S}_{\vbeta}^-(\vmu)$.
\end{proof}

\subsection{Proof of Proposition~\ref{prop:charsingle}}\label{appx:charsingle}
In the particular case when $J=1$, the expression of the characteristic time can be simplified.
\label{sec:proof_char_single}
\charsingle*
\begin{proof}
  The first part of the proof can be obtained using similar argument
  than for Theorem~\ref{th:generic_sample_comp}.  The missing part is
  the simplification of the expression of $T^\star(\vmu)$.

We denote
$\Lambda(\w, \vlambda, \vmu) \df \sum_{a=0}^K w_{a}
d(\mu_{a}, \lambda_{a})$.
Following the reasoning from the proof of
Theorem~\ref{th:generic_sample_comp}
one of the two following conditions needs to be met: (1)
there exists $a \in \mathcal{S}_{\vbeta}(\vmu)$ such that
$\lambda_a < \lambda_0$. (2) there exists
$a \in \mathcal{S}_{\vbeta}^-(\vmu) \df \setc{a \in [K]}{\mu_a < \mu_0}$ such that
$\lambda_a > \lambda_0$.

For this reason, one has
\[
\inf_{\vlambda \in \Alt_{\vbeta}(\vmu)}
\Lambda(\w, \vlambda, \vmu)
= \min \left(
\min_{a \in \mathcal{S}_{\vbeta}(\vmu)}
\inf_{\vlambda: \lambda_a < \lambda_0} \Lambda(\w, \vlambda, \vmu),
\min_{a \in \mathcal{S}_{\vbeta}^-(\vmu)}
\inf_{\vlambda: \lambda_a > \lambda_0} \Lambda(\w, \vlambda, \vmu)
\right) \;.
\]
In this simpler case, it is possible to obtain an explicit formula for this infimum.
We start from
\[
T^\star(\vmu)^{-1} = \sup_{\w \in \Sigma_{K+1}} \min \left(
\min_{a \in \mathcal{S}_{\vbeta}(\vmu)}
\inf_{\vlambda: \lambda_a < \lambda_0} \Lambda(\w, \vlambda, \vmu),
\min_{a \in \mathcal{S}_{\vbeta}^-(\vmu)}
\inf_{\vlambda: \lambda_a > \lambda_0} \Lambda(\w, \vlambda, \vmu)
\right)\;.
\]
Let us focus on the case, $\lambda_a < \lambda_0$ and
fix an index $a \in  \mathcal{S}_{\vbeta}(\vmu)$. $\Lambda$
is always smaller when all the $\lambda_b$ for $b \neq 0$
and $b\neq a$ coincides with $\mu_b$.
This gives,
\[
\min_{a \in \mathcal{S}_{\vbeta}(\vmu)}
\inf_{\vlambda: \lambda_a < \lambda_0} \Lambda(\w, \vlambda, \vmu)
= \min_{a \in \mathcal{S}_{\vbeta}(\vmu)}
\inf_{\vlambda: \lambda_a \leq \lambda_0}
w_0 d(\mu_0, \lambda_0) + w_a d(\mu_a, \lambda_a) \;.
\]
We consider the Lagrangian function,
$L(\lambda_0, \lambda_a, q)= w_0 d(\mu_0, \lambda_0)
+ w_a d(\mu_a, \lambda_a) + q(\lambda_a - \lambda_0)$.
Differentiating with respect to $\lambda_0$ and $\lambda_a$ brings the condition
\[
\lambda_0^\star = \lambda_a^\star = \lambda_{a,0}^\star =
\argmin_{\lambda} w_0 d(\mu_0, \lambda) + w_a d(\mu_a, \lambda)
= \frac{w_0}{w_0 + w_a } \mu_0 + \frac{w_a}{w_0 + w_a} \mu_a  \;.
\]

Recalling, $d_\text{mid}(w_a, \mu_a, w_b, \mu_b)
  \df
  \inf_v w_a d(\mu_a, v) + w_b d(\mu_b, v)$ one has,
\begin{equation}
\label{eq:S_plus}
\min_{a \in \mathcal{S}_{\vbeta}(\vmu)}
\inf_{\vlambda: \lambda_a < \lambda_0} \Lambda(\w, \vlambda, \vmu) =
\min_{a \in \mathcal{S}_{\vbeta}(\vmu)}
 d_\text{mid}(w_0, \mu_0, w_a, \mu_a) \;.
\end{equation}
Solving the optimization program for $a \in \mathcal{S}_{\vbeta}^-(\vmu)$ and under the
constraint $\lambda_a > \lambda_0$, gives
the exact same set of constraints and optimal solution, i.e.
\begin{equation}
\label{eq:S_minus}
\min_{a \in \mathcal{S}_{\vbeta}^-(\vmu)}
\inf_{\vlambda: \lambda_a > \lambda_0} \Lambda(\w, \vlambda, \vmu)
=
\min_{a \in \mathcal{S}_{\vbeta}^-(\vmu)}  d_\text{mid}(w_0, \mu_0, w_a, \mu_a) \;.
\end{equation}
Bringing Equation~\ref{eq:S_plus} and Equation~\ref{eq:S_minus}
together and remarking that
$[K] = \mathcal{S}_{\vbeta}(\vmu) \,
\cup \, \mathcal{S}_{\vbeta}^-(\vmu)$ gives the announced result.
\end{proof}

\subsection{Link between ABC and BAI}
In the particular case, of Gaussian distributions with a known variance $\sigma^2$, for any $\vmu \in \mathcal{L}$, one can easily
create a BAI instance with the same characteristic time as $T^\star(\vmu)$.

\begin{restatable}{lemma}{eqABCBAI}
\label{lemma:ABCBAI}
  Let $\vmu \in \mathcal{L}$ where all the arms are Gaussian
  distributions with known variance $\sigma^2$. Let $\mu_0$ be the
  mean of the control arm. We define $\widetilde{\vmu}$ as follows
  \[
\widetilde{\mu}_k = \left\{
    \begin{array}{ll}
        2 \mu_0 - \mu_k & \mbox{if } \mu_k > \mu_0 \\
        \mu_k & \mbox{otherwise.}
    \end{array}
\right.
\]
  By denoting
  $T^{\star}_{\text{BAI}}(\vmu)$ the characteristic time for the BAI
  problem with a bandit instance $\vmu$, then one has
\[
T^\star(\vmu) = T^{\star}_{\text{BAI}}(\widetilde{\vmu}) \;.
\]
\end{restatable}

\begin{proof}
In the particular case of Gaussian distributions with known variance $\sigma^2$,
easy calculation brings
\[
T^\star(\vmu)^{-1} = \sup_{w \in \Sigma_{K+1}}
\min_{b \neq 0} d_\text{mid}(w_0, \mu_0, w_b, \mu_b)
=
\min_{b \neq 0} \frac{\left(\mu_0 - \mu_b\right)^2}{2}
((1-\alpha_b)^2 w_0 + \alpha_b^2 w_b)
\]
with $\alpha_b = \frac{w_0}{w_b + w_0}$.

On the instance $\tilde{\vmu}$, first note that $\tilde{\mu}_0$ is the best arm by construction.
For an index $k$ such that $\mu_k > \mu_0$, one has
\[
\tilde{\mu}_0 - \tilde{\mu}_k = \mu_0 - (2\mu_0 - \mu_k) = \mu_k - \mu_0 > 0 \;.
\]
For this reason, $\forall k \in [K], \tilde{\mu}_0 > \tilde{\mu}_k$.
Using Lemma 3 from \cite{garivier2016optimal}, by defining
$I_\gamma(\mu_1, \mu_2) \df \gamma d(\mu_1, \gamma \mu_1 + (1-\gamma) \mu_2)
+ (1-\gamma) d(\mu_2, \gamma \mu_1 + (1-\gamma) \mu_2)$, one has
\[
T^{\star}_{\text{BAI}}(\widetilde{\vmu}) =
\sup_{w \in \Sigma_{K+1}}
\min_{b \neq 0} \; (w_0 + w_b) I_{\frac{w_0}{w_0 + w_b}}(\mu_0, \mu_b) \;.
\]
Furthermore, letting $\alpha_b = w_0/(w_0 + w_b)$:
\begin{align*}
(w_0 + w_b) I_{\alpha_b}(\mu_0, \mu_b)
&=
(w_0 + w_b) \left(\alpha_b \frac{(1- \alpha_b)^2 (\mu_0 - \mu_b)^2}{2 \sigma^2}
+(1- \alpha_b)
\frac{\alpha_b^2 (\mu_0 - \mu_b)^2}{2 \sigma^2} \right)
\\
& =
\frac{(\mu_0 - \mu_b)^2}{2 \sigma^2}
\left( (1 - \alpha_b)^2 w_0 + \alpha_b^2 w_b \right)\;.
\end{align*}
Plugging this in the expression of $T^{\star}_{\text{BAI}}(\widetilde{\vmu})$ gives the announced result.
\end{proof}

\section{Results for specific modes of interaction}
\label{app:interactions}

\subsection{Agnostic mode}
\begin{lemma}
\label{lemma:agnostic}
For any agnostic policy where $A_t$ is chosen knowing
$\mathcal{F}_{t-1}$ but independently from $I_t$,
when defining
$N_{a,j}(t) = \sum_{s=1}^t \mathds{1}(A_s = a \,
\cap \, I_s = j)$
and $N_{a}(t) = \sum_{s=1}^t \mathds{1}(A_s = a)$,
then
$$
\forall a \in \{0, \dotsc, K\}, \forall j \in \{1, \dotsc, J \},
\forall t \geq 1, \quad
\mathbb{E}_\vmu[N_{a,j}(t) ] =
\alpha_j \mathbb{E}_\vmu[N_a(t)]
$$
\end{lemma}

\begin{proof}
\begin{align*}
\mathbb{E}_\vmu[N_{a,j}(t) ] &= \sum_{s=1}^t \mathbb{P}(A_s = a \, \cap \, I_s = j)
= \sum_{s=1}^t \mathbb{P}_\vmu(A_s = a | I_s = j) \mathbb{P}(I_s = j) \\
&=\sum_{s=1}^t \alpha_j \mathbb{P}_\vmu(A_s = a | I_s = j)
= \sum_{s=1}^t \alpha_j \mathbb{P}_\vmu(A_s = a ) \\
&= \alpha_j\mathbb{E}_\vmu[N_{a}(t) ]\;,
\end{align*}
where in the third equality, we have used that the action $A_t$ is
selected independently from the population indicator $I_t$.
\end{proof}

\subsection{Proportional mode}
\begin{lemma}
\label{lemma:proportional}
For any proportional policy where $A_t$ is chosen knowing $\mathcal{F}_{t-1}$ and $I_t$,
when defining
$N_{a,j}(t) = \sum_{s=1}^t \mathds{1}(A_s = a \, \cap \, I_s = j)$ and $N_{a}(t) = \sum_{s=1}^t \mathds{1}(A_s= a)$, then
$$
\forall j \in \{1, \dotsc, J \},
\forall t \geq 1,
\quad
\sum_{a=0}^{K} \mathbb{E}_\vmu[N_{a,j}(t) ] = \alpha_j t \;.
$$
\end{lemma}

\begin{proof}
\begin{align*}
\sum_{a=0}^K \mathbb{E}_\vmu[N_{a,j}(t) ] &= \sum_{s=1}^t
\sum_{a=0}^K
 \mathbb{E}_\vmu\left[ \mathds{1}(I_s = j)
\mathds{1}(A_s = a) \right] =
\sum_{s=1}^t
 \mathbb{E}_\vmu\left[ \mathds{1}(I_s = j)
 \sum_{a=0}^K
\mathds{1}(A_s = a) \right] \\
&= \sum_{s=1}^t \mathbb{P}_\vmu(I_s = j)
=\alpha_j t \;.
\end{align*}
\end{proof}

\subsection{Oblivious mode}
\label{app:oblivious}
In the oblivious mode, the subpopulations can not be observed by the learner. In this case,
we have
\begin{enumerate}
\item $\mathbb{E}\left[ X_t | A_t = a\right] = \sum_{i=1}^J \alpha_i \mu_{a,i}$ \;,
\item $X_t | A_t = a \sim \sum_{i=1}^J \alpha_i \nu_{a,i}$ \;.
\end{enumerate}
While with observable subpopulations the distributions are entirely characterized by their means, this is no longer the case with mixture distributions.
In particular, this requires defining a different alternative.
\[
\textnormal{Alt}(\nu) \df \{ \nu' : \forall a, \nu_a' = \sum_{i=1}^J \alpha_i \nu_{a,i}' \; \textnormal{with}  \; \nu'_{a,i} \in \mathcal{P}
\; \textnormal{and} \; \mathcal{S}(\nu') \neq \mathcal{S}(\nu) \; \} \;.
\]
\begin{proposition}
Let $\delta \in (0,1)$ and $\vbeta \in \mathbb{R}^J$. For any \textit{oblivious} strategy satisfying Equation~\ref{eq:p-value}
and any $\vmu \in \mathcal{L}$, the characteristic time satisfies
\begin{equation*}
\mathbb{E}_\vmu[\tau_{\delta}] \geq T_{\textnormal{oblivious}}^\star(\vmu) \,\textnormal{kl}(\delta, 1- \delta)
\quad
\textnormal{and}
\quad
\liminf_{\delta \to 0} \frac{\mathbb{E}_\vmu[\tau_\delta]}{\ln(1/\delta)}
\geq T_{\textnormal{oblivious}}^\star(\vmu) \;.
\end{equation*}
where
\begin{align}
T_{\textnormal{oblivious}}^\star(\vmu)^{-1} &= \sup_{\w \in \Sigma_{K+1}}
\inf_{\nu' \in \, \textnormal{\Alt}(\nu)}
\sum_{a=0}^K w_{a} \textnormal{KL}\left( \sum_{i=1}^J \alpha_i \nu_{a,i}, \sum_{i=1}^J \alpha_i \nu_{a,i}' \right)
 \;.
\end{align}
Furthermore,
\[
\forall \vmu \in \mathcal{L}, \quad
T_{\textnormal{oblivious}}^\star(\vmu) \geq T_{\textnormal{agnostic}}^\star(\vmu) \;.
\]
\label{prop:sample_comp_oblivious}
\end{proposition}

\begin{proof}
 Using the transportation lemma from \cite{kaufmann2016complexity}
 we have for any safely calibrated oblivious policy
\[
\forall \nu' \in \Alt(\nu), \;
\sum_{a=0}^K \mathbb{E}_\vmu[N_{a}(\tau_\delta)]
\textnormal{KL}\left( \sum_{i=1}^J \alpha_i \nu_{a,i}, \sum_{i=1}^J \alpha_i \nu_{a,i}' \right)
\geq \textnormal{kl}(\delta, 1- \delta) \;.
\]
Therefore,
\begin{align*}
\textnormal{kl}(\delta, 1- \delta)
& \leq
\inf_{\nu' \in \Alt(\nu)}
\sum_{a=0}^K \mathbb{E}_\vmu[N_{a}(\tau_\delta)]
\textnormal{KL}\left( \sum_{i=1}^J \alpha_i \nu_{a,i}, \sum_{i=1}^J \alpha_i \nu_{a,i}' \right)
\\
& =  \mathbb{E}_\vmu[\tau_\delta]
\inf_{\nu' \in \Alt(\nu)}
\sum_{a=0}^K \frac{\mathbb{E}_\vmu[N_{a}(\tau_\delta)]}{\mathbb{E}_\vmu[\tau_\delta] }
\textnormal{KL}\left( \sum_{i=1}^J \alpha_i \nu_{a,i}, \sum_{i=1}^J \alpha_i \nu_{a,i}' \right)
\\
& \leq
\mathbb{E}_\vmu[\tau_\delta]
\sup_{\w \in  \Sigma_{K+1}}
\inf_{\nu' \in \Alt(\nu)}
\sum_{a=0}^K w_a
\textnormal{KL}\left( \sum_{i=1}^J \alpha_i \nu_{a,i}, \sum_{i=1}^J \alpha_i \nu_{a,i}' \right)  \;.
\end{align*}
Using
$\textnormal{kl}(\delta, 1-\delta) \sim \ln(1/\delta)$ when $\delta$
tends to 0 gives the first result.

Using the joint convexity of the KL divergence one gets
\[
\textnormal{KL}\left( \sum_{i=1}^J \alpha_i \nu_{a,i}, \sum_{i=1}^J \alpha_i \nu_{a,i}' \right)
\leq
\sum_{i=1}^J \alpha_i \textnormal{KL}(\nu_{a,i}, \nu'_{a,i}) \;.
\]

Assuming that the mean of $\nu'_{a,i} = \lambda_{a,i}$ and recalling that for distributions in $\mathcal{P}$, one has
$\textnormal{KL}(\nu_{a,i}, \nu'_{a,i}) = d(\mu_{a,i}, \lambda_{a,i})$, we deduce,
\begin{align*}
T_{\textnormal{oblivious}}^\star(\vmu)^{-1} &= \sup_{w \in \Sigma_{K+1}}
\inf_{\nu' \in \, \textnormal{\Alt}(\nu)}
\sum_{a=0}^K w_{a} \textnormal{KL}\left( \sum_{i=1}^J \alpha_i \nu_{a,i}, \sum_{i=1}^J \alpha_i \nu_{a,i}' \right)
\\
&
\leq
\sup_{w \in \Sigma_{K+1}}
\inf_{\vlambda \in \, \textnormal{\Alt}(\vmu)}
\sum_{a=0}^K \sum_{i=1}^J \alpha_i w_{a} d(\mu_{a,i}, \lambda_{a,i}) \\
& = T_{\textnormal{agnostic}}^\star(\vmu)^{-1} \;.
\end{align*}
\end{proof}

Except in the case of Bernoulli distributions --where the mixture is also a Bernoulli distribution--, finding a strategy that matches $T_{\textnormal{oblivious}}^\star(\vmu)^{-1}$ is a hard task. However, one may use the following lemma to treat the mixture in a sub-optimal way, based on the fact that it exhibits sub-gaussian behavior.

\begin{lemma}[Sub-gaussianity of mixture]
  For each $\mu \in \mathbb R$, assume that $\nu_\mu$ is a distribution on $\mathbb R$ with mean $\ex_{X \sim \nu_\mu}[X] = \mu$ that is $\sigma^2$-sub-Gaussian, meaning that $\ex_{X \sim \nu_\mu}\sbr*{e^{\lambda (X - \mu)}} \le e^{\sigma^2 \lambda^2/2}$ for any $\lambda \in \mathbb R$. Further let $\alpha(\mu)$ be a prior on $\mu$ with mean $m$ that is itself $\eta^2$-sub-Gaussian, meaning that $\ex_{\mu \sim \alpha} \sbr*{e^{\lambda(\mu -m)}} \le e^{\lambda^2 \eta^2/2}$. Then the mixture distribution $Q = \ex_{\mu \sim \alpha} \sbr*{\nu_\mu}$ is $\sigma^2+\eta^2$ sub-Gaussian.
\end{lemma}
\begin{proof}
  The mixture distribution obviously has mean $\ex_{X \sim Q}[X] = m$ and
\begin{align*}
  \ex_{X \sim Q} \sbr*{e^{\lambda (X-m)}}
  &~=~
    \ex_{\mu \sim \alpha} \sbr*{
    e^{\lambda(\mu - m)}
    \ex_{X \sim \nu_\mu} \sbr*{
    e^{\lambda (X - \mu)}
    }
    }
  \\
  &~\le~
    \ex_{\mu \sim \alpha} \sbr*{ e^{\lambda (\mu-m)} } e^{\sigma^2 \lambda^2/2}
  \\
  &~\le~
    e^{(\sigma^2 +\eta^2) \lambda^2/2} .
\end{align*}
\end{proof}
In particular, if $\alpha$ is supported on $[\pm M]$, then $\alpha$ is $M^2$ sub-Gaussian, and hence $Q$ is $(\sigma^2+M^2)$ sub-Gaussian.

\section{Optimal allocations in the Gaussian case for $K=1$ (A/B testing)}
\label{app:gaussian}

\begin{lemma}
\label{lemma:alt_simplification}
When $K=1$ with Gaussian distributions such that
$\nu_{a,i} = \mathcal{N}(\mu_{a,i}, \sigma_{a,i}^2)$ the following holds
\[
\inf_{\vlambda: \lambda_0 = \lambda_1} \sum_{i=1}^J w_{0,i} d(\mu_{0,i},\lambda_{0,i})
+ \sum_{i=1}^J w_{1,i} d(\mu_{1,i},\lambda_{1,i})
=
\frac{\Delta_1^2}{2 \sum_{i=1}^J \beta_i^2 \left( \frac{\sigma_{0,i}^2}{w_{0,i}}
+ \frac{\sigma_{1,i}^2}{w_{1,i}} \right)} \;.
\]
\end{lemma}

\begin{proof}
One has for $b \in \{0,1\}$,
\[
d(\mu_{b,i}, \lambda_{b,i}) = \frac{(\lambda_{b,i} - \mu_{b,i})^2}{2 \sigma_{b,i}^2} \;.
\]
Using the result from
Theorem~\ref{th:generic_sample_comp} for the case $K=1$, the following holds
\[
\inf_{\vlambda \in \Alt_{\vbeta}(\vmu)} \sum_{a=0}^1 \sum_{i=1}^J w_{a,i} d(\mu_{a,i}, \lambda_{a,i}) =
\min_{\vlambda \in \mathcal{L}: \lambda_0 = \lambda_1}
\sum_{a=0}^1 \sum_{i=1}^J w_{a,i} d(\mu_{a,i},\lambda_{a,i}) \;.
\]
We introduce
\begin{equation*}
L(\lambda_0, \lambda_1, q) = \sum_{i=1}^J w_{0,i} \frac{(\lambda_{0,i} - \mu_{0,i})^2}{2 \sigma_{0,i}^2}
+
\sum_{i=1}^J w_{1,i} \frac{(\lambda_{1,i} - \mu_{1,i})^2}{2 \sigma_{1,i}^2} +
q \left( \sum_{i=1}^J \beta_i(\lambda_{0,i} - \lambda_{1,i})\right) \;.
\end{equation*}
One has,
\[
\min_{\vlambda \in \mathcal{L}: \lambda_0 = \lambda_1}
\sum_{a=0}^1 \sum_{i=1}^J w_{a,i} d(\mu_{a,i},\lambda_{a,i})  =
\sup_{q \in \mathbb R} \inf_{\vlambda \in \mathcal{L}} L(\lambda_0, \lambda_1, q)
\;.
\]
Differentiating with respect to $\lambda_{0,i}$ and $\lambda_{1,i}$ brings the conditions
\[
\lambda_{0,i} = \mu_{0,i} - \frac{q \beta_i \sigma_{0,i}^2}{w_{0,i}}
\quad
\textnormal{and}
\quad
\lambda_{1,i} = \mu_{1,i} + \frac{q \beta_i \sigma_{1,i}^2}{w_{1,i}} \;.
\]
Plugging these values back in $L$ gives the function
\[
f(q) = - \frac{q^2}{2}\sum_{i=1}^J  \beta_i^2 \left( \frac{\sigma_{0,i}^2}{w_{0,i}} +
\frac{\sigma_{1,i}^2}{w_{1,i}} \right)
+ q \sum_{i=1}^J \beta_i (\mu_{0,i} - \mu_{1,i}) \;.
\]
Easy calculations show that the maximum of the function $f$ is attained for
\[
q^\star = \frac{\sum_{i=1}^J \beta_i(\mu_{0,i} - \mu_{1,i})}{\sum_{i=1}^J \beta_i^2
\left(  \frac{\sigma_{0,i}^2}{w_{0,i}} +
\frac{\sigma_{1,i}^2}{w_{1,i}}  \right)} \;.
\]
Plugging this value back in the expression of $f$,
\[
f(q^\star) =  \frac{\Delta_1^2}{ 2 \sum_{i=1}^J \beta_i^2 \left( \frac{\sigma_{0,i}^2}{w_{0,i}}
+ \frac{\sigma_{1,i}^2}{w_{1,i}} \right)}\;.
\]
\end{proof}

\lemmagaussfirst*

\begin{proof}
\item
\paragraph{Agnostic mode}
From the Lemma~\ref{lemma:agnostic}, we have
\begin{align*}
T_{\textnormal{agnostic}}^\star(\vmu)^{-1}
&=
\sup_{\w \in \mathcal{C}_{\textnormal{agnostic}}}
\inf_{\vlambda \in \textnormal{Alt}_{\vbeta}(\vmu)} \sum_{a=0}^1
\sum_{i=1}^J w_{a,i} d(\mu_{a,i}, \lambda_{a,i})
\\
&=
\sup_{\w \in \mathcal{C}_{\textnormal{agnostic}}}
\inf_{\vlambda: \lambda_0 = \lambda_1} \sum_{a=0}^1
\sum_{i=1}^J w_{a,i} d(\mu_{a,i}, \lambda_{a,i})
\quad \textnormal{(Theorem~\ref{th:generic_sample_comp})}
\\
&=
\sup_{\w \in \mathcal{C}_{\textnormal{agnostic}}}
\frac{\Delta_1^2}{2 \sum_{i=1}^J \beta_i^2 \left( \frac{\sigma_{0,i}^2}{w_{0,i}}
+ \frac{\sigma_{1,i}^2}{w_{1,i}} \right)}
\quad
\textnormal{(Lemma~\ref{lemma:alt_simplification})} \;.
\end{align*}

$\w \in \mathcal{C}_{\textnormal{agnostic}}$ implies $w_{a,i} = \alpha_i u_a$
with $(u_0,\dotsc, u_K) \in \Sigma_{K+1}$.
For this reason,
\[
\w^\star =  \argmin_{\vu: u_0 + u_1 = 1}\sum_{i=1}^J \frac{\beta_i^2}{\alpha_i}
\left(
\frac{\sigma_{0,i}^2}{u_{0}}
+ \frac{\sigma_{1,i}^2}{u_{1}}
\right) \;.
\]

We let $c_a \df \sum_{i=1}^J \frac{\beta_i^2 \sigma_{a,i}^2 }{\alpha_i}$ for $a \in \{0,1\}$.
Plugging $u_1 = 1- u_0$ in the previous expression  and differentiating with respect to
$u_0$ brings the condition
\[
u_0^2 + 2 u_0 \frac{c_0}{c_1 - c_0} - \frac{c_0}{c_1-c_0} \;.
\]
Solving this polynomial and using that $\vu \in \Sigma_2$ gives the unique solution
\[
u_0^\star= \frac{\sqrt{c_0}}{\sqrt{c_0} + \sqrt{c_1}} \;.
\]
Implying,
\[
w_{0,i}^\star = \alpha_i \frac{\sqrt{\sum_{i=1}^J \frac{\beta_i^2 \sigma_{0,i}^2 }{\alpha_i}}}{\sqrt{\sum_{i=1}^J \frac{\beta_i^2 \sigma_{0,i}^2 }{\alpha_i}} + \sqrt{\sum_{i=1}^J \frac{\beta_i^2 \sigma_{1,i}^2 }{\alpha_i}}}
\quad
\textnormal{and}
\quad
w_{1,i}^\star = \alpha_i \frac{\sqrt{\sum_{i=1}^J \frac{\beta_i^2 \sigma_{1,i}^2 }{\alpha_i}}}{\sqrt{\sum_{i=1}^J \frac{\beta_i^2 \sigma_{0,i}^2 }{\alpha_i}} + \sqrt{\sum_{i=1}^J \frac{\beta_i^2 \sigma_{1,i}^2 }{\alpha_i}}} \;.
\]

With those values,
\[
T_{\textnormal{agnostic}}^\star(\vmu)
=
\frac{2\left( \sqrt{\sum_{i=1}^J \frac{\beta_i^2 \sigma_{0,i}^2 }{\alpha_i}}
+
\sqrt{\sum_{i=1}^J \frac{\beta_i^2 \sigma_{1,i}^2 }{\alpha_i}}
\right)^2}{\Delta_1^2} \;.
\]

\item
\paragraph{Proportional mode}
Following the same line of proof, gives
\[
T_{\textnormal{prop}}^\star(\vmu)^{-1}
=
\sup_{\w \in \mathcal{C}_{\textnormal{prop}}}
\frac{\Delta_1^2}{2 \sum_{i=1}^J \beta_i^2 \left( \frac{\sigma_{0,i}^2}{w_{0,i}}
+ \frac{\sigma_{1,i}^2}{w_{1,i}} \right)}  \;.
\]
The main difference is now on the constraints on the weights. In the proportional mode,
following Lemma~\ref{lemma:proportional}, $\forall i \leq J,
\sum_{a=0}^1 w_{a,i} = \alpha_i$.
We consider the Lagrangian function:

\[
L(w_0, w_1, q_1, \dotsc, q_J) =
\sum_{i=1}^J \beta_i^2
\left( \frac{\sigma_{0,i}^2}{w_{0,i}} + \frac{\sigma_{1,i}^2}{w_{1,i}}
\right)
+ \sum_{i=1}^J q_i \left(\sum_{a \in \{0,1\}} w_{a,i} - \alpha_i \right) \;.
\]
Differentiating with respect to $w_{0,i}$ and $w_{1,i}$ gives the constraints:
\[
\frac{-\beta_i^2 \sigma_{0,i}^2}{w_{0,i}^2} + q_i = 0
\quad
\textnormal{and}
\quad
\frac{-\beta_i^2 \sigma_{1,i}^2}{w_{1,i}^2} + q_i = 0 \;.
\]
From which we can deduce
\[
\frac{w_{0,i}}{\sigma_{0,i}} =  \frac{w_{1,i}}{\sigma_{1,i}} \;.
\]
From $w_{0,i} + w_{1,i} = \alpha_i$, we deduce,
\[
q_i^\star =\frac{\beta_i^2(\sigma_{0,i} + \sigma_{1,i})^2}{\alpha_i^2} \;.
\]
Plugging this value in the first constraint gives
\[
w_{0,i}^\star = \alpha_i \frac{\sigma_{0,i}}{\sigma_{0,i} + \sigma_{1,i}}
\quad
\textnormal{and}
\quad
w_{1,i}^\star = \alpha_i \frac{\sigma_{1,i}}{\sigma_{0,i} + \sigma_{1,i}}
\;.
\]
Using those weights,
\[
T_{\textnormal{prop}}^\star(\vmu)
=
 \frac{
2 \sum_{i=1}^J \frac{\beta_i^2}{\alpha_i} (\sigma_{0,i} + \sigma_{1,i})^2}
{\Delta_1^2} \;.
\]
\item
\paragraph{Active mode}
Following the proof of Proposition~\ref{proposition_others_weights}, one has
\begin{equation}
\label{eq:nulos}
T_{\textnormal{active}}^\star(\vmu)^{-1}
=
\sup_{\w \in \Sigma_{(K+1)J}}
\frac{\Delta_1^2}{2 \sum_{i=1}^J \beta_i^2 \left( \frac{\sigma_{0,i}^2}{w_{0,i}}
+ \frac{\sigma_{1,i}^2}{w_{1,i}} \right)}  \;.
\end{equation}
Using the constraint $\w \in \Sigma_{(K+1)J}$, one gets
\begin{equation}
\label{eq:w_1_J}
w_{1,J} = 1 - \sum_{a= \{0,1\}} \sum_{i=1}^{J-1} w_{a,i} - w_{0,J} \;.
\end{equation}
We need to minimize the function
(where $w_{1,J}$ has been replaced by the
expression from Equation~\ref{eq:w_1_J})
\[
f(\w) = \sum_{a= \{0,1\}} \sum_{i=1}^{J-1} \beta_i^2\frac{\sigma_{a,j}^2}{w_{a,j}}
+ \beta_J^2 \frac{\sigma_{0,J}^2}{w_{0,J}}
+ \beta_J^2 \frac{\sigma_{1,J}^2}{1 - \sum_{a=0}^1 \sum_{i=1}^{J-1} w_{a,i} - w_{0,J}} \;.
\]
For $i \leq J-1$, taking the derivative with respect to
$w_{0,i}$ and $w_{1,i}$ gives the following constraints
\[
\beta_i^2 \sigma_{0,i}^2 \left(1 - \sum_{a=0}^1 \sum_{i=1}^{J-1} w_{a,i} - w_{0,J} \right)^2
= \beta_J^2 \sigma_{1,J}^2 w_{0,i}^2  \;,
\]
\[
\beta_j^2 \sigma_{1,i}^2 \left(1 - \sum_{a=0}^1 \sum_{i=1}^{J-1} w_{a,i} - w_{0,J} \right)^2
= \beta_J^2 \sigma_{1,J}^2 w_{1,i}^2  \;.
\]
From which we deduce
\begin{equation}
\label{eq:frac_weights}
\forall i \leq J-1, \quad \frac{w_{0,i}}{\sigma_{0,i}} = \frac{w_{1,i}}{\sigma_{1,i}} \;.
\end{equation}
Differentiating with respect to $w_{1,J}$ gives
\[
\sigma_{0,J} \left( 1 - \sum_{a=0}^1
\sum_{i=1}^{J-1} w_{a,i} - w_{0,J} \right) = \sigma_{1,J} w_{0,J} \;.
\]
Rearranging and using Equation~\ref{eq:frac_weights} gives,
\begin{equation}
\label{eq:w_0_J}
w_{0,J} = \frac{\sigma_{0,J}}{\sigma_{0,J} + \sigma_{1,J}} - \sum_{i=1}^{J-1}
\frac{w_{0,i}}{\sigma_{0,i}} \frac{\sigma_{0,i} + \sigma_{1,i}}{\sigma_{0,J} + \sigma_{1,J}}
\sigma_{0,J} \;.
\end{equation}

Using Equation~\ref{eq:frac_weights} and Equation~\ref{eq:w_0_J}, we define the function
\[
g(w_{0,1}, \dotsc, w_{0,J-1})
= \sum_{i=1}^{J-1} \beta_i^2 \frac{\sigma_{0,i}}{w_{0,i}}
(\sigma_{0,i} + \sigma_{1,i})
+ \frac{\beta_J^2 (\sigma_{0,J} + \sigma_{1,J})^2}{1- \sum_{i=1}^{J-1}
\frac{w_{0,i}(\sigma_{0,i} + \sigma_{1,i})}{\sigma_{0,i}} } \;.
\]
Differentiating with respect to $w_{0,i}$ for $i \leq J-1$ brings
\begin{equation}
\label{eq:solving_active_mode}
\forall i \leq J-1, \quad
\frac{| \beta_i| \sigma_{0,i}}{\sigma_{0,J} + \sigma_{1,J}}
\left( 1- \sum_{i=1}^{J-1} \frac{w_{0,i}}{\sigma_{0,i}} (\sigma_{0,i} + \sigma_{1,i})
\right)
=
| \beta_J | w_{0,i} \;.
\end{equation}
Multiplying both sides of this equation by $(\sigma_{0,i} + \sigma_{1,i})/\sigma_{0,i}$
and summing for $i \leq J-1$,
\[
\sum_{i=1}^{J-1} \frac{w_{0,i}}{\sigma_{0,i}} \left( \sigma_{0,i} + \sigma_{1,i} \right)
= \frac{\sum_{i=1}^{J-1} |\beta_i| (\sigma_{0,i} + \sigma_{1,i})}{\sum_{i=1}^{J}
|\beta_i| (\sigma_{0,i} + \sigma_{1,i}) } \;.
\]
Plugging this value back in Equation~\ref{eq:solving_active_mode} one has:
\[
\forall i \leq J-1, \quad
w_{0,i} = \frac{| \beta_i| \sigma_{0,i}}{\sum_{i=1}^J |\beta_i|(\sigma_{0,i} + \sigma_{1,i})}
\;.
\]
From Equation~\ref{eq:frac_weights}, we deduce,
\[
\forall i \leq J-1, \quad
w_{1,i} = \frac{| \beta_i| \sigma_{1,i}}{\sum_{i=1}^J |\beta_i|(\sigma_{0,i} + \sigma_{1,i})}
\;.
\]
We obtain the value of $w_{0,J}$ using Equation~\ref{eq:w_0_J} and that of
$w_{1,J}$ using Equation~\ref{eq:w_1_J}.
Plugging those weights in the expression given by
Equation~\ref{eq:nulos} yields the characteristic time.
\end{proof}

\section{General form of the optimal allocation in the Gaussian case}
\label{app:optimization}
\optimsimple*
\begin{proof}
From Lemma~\ref{lemma:alt_simplification}, when the distribution are Gaussian with a known variance
$\sigma^2$ one has
\[
T_{\textnormal{active}}^\star(\vmu)^{-1}
=
\sup_{\w \in \Sigma_{(K+1)J}}
\min_{b \neq 0}
\frac{\Delta_1^2}{2 \sum_{i=1}^J \beta_i^2 \left( \frac{\sigma^2}{w_{0,i}}
+ \frac{\sigma^2}{w_{b,i}} \right)} \;.
\]
Using the same continuity argument than in \cite{garivier2016optimal}, we know that the supremum
of $\w$ is attained and is indeed a maximum.
Let
\[
\Lambda_b(v, w) \df \frac{\Delta_b^2}{2 \sum_{i=1}^J \beta_i^2 \left( \frac{\sigma^2}{v_i}
+ \frac{\sigma^2}{w_i} \right)} \;.
\]
Then
\begin{align*}
\max_{\w \in \Sigma_{(K+1)J}}
\min_{b \neq 0} \;
\Lambda_b(w_0, w_b)
&=
\max_{\substack{u \in \Sigma_{K+1} \\ \forall a, \sum_{i=1}^J w_{a,i} = u_a}}
\min_{b \neq 0}
\;
\Lambda_b(w_0, w_b) \\
&=
\max_{u \in \Sigma_{K+1}}
\max_{\substack{w \in \Sigma_{(K+1)J} \\ \forall a, \sum_{i} w_{a,i} = u_a}}
\min_{b \neq 0}
\;
\Lambda_b(w_0, w_b)
\\
& \leq
\max_{u \in \Sigma_{K+1}}
\min_{b \neq 0}
\max_{\substack{w \in \Sigma_{(K+1)J} \\ \forall a, \sum_{i} w_{a,i} = u_a}}
\Lambda_b(w_0, w_b)
\quad
\textnormal{(Max-min inequality)} \;.
\end{align*}
Let $b \neq 0$,
\[
\max_{\substack{w \in \Sigma_{(K+1)J} \\ \forall a, \sum_{i} w_{a,i} = u_a}}
\Lambda_b(w_0, w_b)
= \max_{\substack{w \in \Sigma_{(K+1)J} \\ \forall a, \sum_{i} w_{a,i} = u_a}}
\frac{\Delta_b^2}{2 \sum_{i=1}^J \beta_i^2 \left( \frac{\sigma^2}{w_{0,i}}
+ \frac{\sigma^2}{w_{b,i}} \right)}  \;.
\]

Equivalently, we are interested in
\[
\min_{\substack{w \in \Sigma_{(K+1)J} \\ \forall a, \sum_{i} w_{a,i} = u_a}}
\sum_{i=1}^J \beta_i^2 \left( \frac{\sigma^2}{w_{0,i}}
+ \frac{\sigma^2}{w_{b,i}} \right) \;.
\]
We introduce the associated Lagrangian function
\[
f(w, q) = \sum_{i=1}^J \beta_i^2 \left( \frac{1}{w_{0,i}}
+ \frac{1}{w_{b,i}} \right) + \sum_{a=0}^K q_a \left(\sum_{i=1}^J w_{a,i} - u_a \right) \;.
\]
Taking the derivative with respect to $w_{0,i}$ and ${w_{b,i}}$ for the different values of $i$
yields
\[
w_{0,i} = \frac{|\beta_i|}{\sqrt{q_0}}
\quad
\textnormal{and}
\quad
w_{b,i} = \frac{|\beta_i |}{\sqrt{q_b}} \;.
\]
Summing over $i$ implies that
\[
\sqrt{q_0} = \frac{\sum_{i=1}^J | \beta_i |}{u_0}
\quad
\textnormal{and}
\quad
\sqrt{q_b} = \frac{\sum_{i=1}^J | \beta_i |}{u_b} \;,
\]
and plugging the above in the expression of the weights yields
\[
w_{0,i} = \frac{|\beta_i|}{\sum_{i=1}^J | \beta_i |} u_0
\quad
\textnormal{and}
\quad
w_{b,i} = \frac{|\beta_i|}{\sum_{i=1}^J | \beta_i |} u_b \;.
\]
In particular,
\begin{equation}
\max_{\substack{w \in \Sigma_{(K+1)J} \\ \forall a, \sum_{i} w_{a,i} = u_a}}
\frac{\Delta_b^2}{2 \sum_{i=1}^J \beta_i^2 \left( \frac{\sigma^2}{w_{0,i}}
+ \frac{\sigma^2}{w_{b,i}} \right)} =
\frac{\Delta_b^2}{2 \sigma^2 \left(\sum_{i=1}^J |\beta_i|\right)^2 \left( \frac{1}{u_{0}}
+ \frac{1}{u_{b}} \right)} \;,
\end{equation}
yielding
\[
\max_{\w \in \Sigma_{(K+1)J}}
\min_{b \neq 0} \frac{\Delta_b^2}{2 \sum_{i=1}^J \beta_i^2 \left( \frac{\sigma^2}{w_{0,i}}
+ \frac{\sigma^2}{w_{b,i}} \right)}
\leq
\frac{1}{2 \sigma^2 \left(\sum_{i=1}^J |\beta_i|\right)^2}
\max_{u \in \Sigma_{(K+1)J}}
\min_{b \neq 0}  \frac{\Delta_b^2}{\left( \frac{1}{u_{0}}
+ \frac{1}{u_{b}} \right)} \;.
\]

On the other hand, letting $w_{a,i} = \frac{|\beta_i|}{\sum_{i=1}^J | \beta_i |} u_a$ with $\sum_{a} u_a =1$
we have
\[
\frac{1}{2 \sigma^2 \left(\sum_{i=1}^J |\beta_i|\right)^2}
\max_{u \in \Sigma_{(K+1)J}}
\min_{b \neq 0}  \frac{\Delta_b^2}{\left( \frac{1}{u_{0}}
+ \frac{1}{u_{b}} \right)}
\leq
\max_{\w \in \Sigma_{(K+1)J}}
\min_{b \neq 0} \frac{\Delta_b^2}{2 \sum_{i=1}^J \beta_i^2 \left( \frac{\sigma^2}{w_{0,i}}
+ \frac{\sigma^2}{w_{b,i}} \right)} \; ,
\]
showing that the two optimization programs are equivalent and
that when denoting
\[
(u_0^\star, \dotsc, u_K^\star) =
\argmax_{u \in \Sigma_{(K+1)J}}
\min_{b \neq 0}  \frac{\Delta_b^2}{\left( \frac{1}{u_{0}}
+ \frac{1}{u_{b}} \right)} \;,
\]
one has
\[
\forall a \in \{0, \dotsc, K \}, \; \forall i \leq J, \; w_{a,i}^\star = u_a^\star \frac{|\beta_i|}{\sum_{i=1}^J | \beta_i |} \;.
\]
This corresponds to the optimal allocation strategy in the \textit{active} mode.
Recalling that when $\alpha = \beta$, the optimal weights for the \textit{active} mode satisfy both
$\mathcal{C}_{\textnormal{prop}}$ and $\mathcal{C}_{\textnormal{agnostic}}$ completes the proof.
\end{proof}

\section{Asymptotic Optimality: Proof of Theorem~\ref{thm:asym.opt} }
\label{appx:alg.works}

In this section we show that T-a-S with C-tracking and a certain threshold $\beta(t,\delta)$ is safely calibrated and asymptotically optimal. This is an important sanity check to validate our approach theoretically. Note that for the experimental validation we have explored a practically appealing variant of this algorithm: we employ an iterative scheme to approximate $\w^*(\hat \vmu(t))$, use D-tracking, and stylise the threshold.

Safe calibration follows from the definition of the recommendation rule (we report the answer $\mathcal{S}_{\vbeta}(\hat \vmu(t))$  at the empirical estimate $\hat \vmu(t)$ of the bandit instance), together with the computation of the risk assessment $\hat \delta_t$. It does not depend on the sampling rule. Our confidence level $\hat \delta_t$ is obtained by inverting the threshold $\beta(t, \delta)$ at the GLR statistic \eqref{eq:GLR}. Safe calibration then follows from an anytime-valid GLR deviation inequality with boundary $\beta(t, \delta)$. We refer to \cite[Proposition~23]{mixmart} for a boundary that is, in case of the ABC-S problem, of order $\ln \frac{1}{\delta} + K + 2 J \cdot O(\ln \ln \frac{t}{\delta})$.

It remains to argue that the T-a-S sampling rule converges to the oracle weights. The original T-a-S proof for the BAI problem is due to \cite[Theorem~14]{garivier2016optimal}. An upgrade to any single-answer problem, including our ABC-S, is due to \cite{multiple.answers}. For active mode, their theorem applies directly, while for agnostic mode it applies with the pair $(I_t, X_t)$ regarded as the observation. We get:

\begin{theorem}[{\cite[Theorems~7 and~10]{multiple.answers}}]
  For all ABC-S instances $\vmu \in \mathcal{L}$ in active mode and agnostic mode, Track-and-Stop with C-tracking and stopping threshold $\beta(t,\delta) = \ln(t^2/\delta) + O(1)$ is $\delta$-correct with asymptotically optimal sample complexity.
\end{theorem}

In \textit{proportional} mode, we have the additional constraint that the learner chooses its arm in response to seeing (but not controlling) the subpopulation $I_t$. Still, the tracking convergence result \cite[Lemma~6]{multiple.answers} goes through, upon observing that the empirical distribution of $I_t$ converges to $\valpha$ by the law of the large numbers, and hence our conditional tracking (see ``sampling rule'' in Section~\ref{sec:algs}) adds the right conditional to the right marginal. All in all, the computed joint weights converge to the joint
$\w^*_\text{prop}(\vmu)$, and tracking makes the sampling proportions also converge there.

We conclude with a remark on our use of D-tracking. Recall that D-tracking is the idea of advancing $N_a(t)$ towards $t$ times the most current oracle weights, i.e.\ $t w^*_a(\hat \vmu(t))$, while C-tracking makes $N_a(t)$ advance towards the sum of encountered oracle weights, i.e.\ $\sum_{s=1}^t w^*_a(\hat \vmu(s))$. As argued in \cite[Appendix~E]{purex.games}, D-tracking can fail to make $N_a(t)/t$ converge to $w^*_a(\vmu)$.  However, this requires that the maximiser of the lower bound problem is not unique at $\vmu$ (as we are maximising a concave function, the set of maximisers is always convex). Here we argue that such a situation does not occur for the ABC-S problem. To see why, we argue that the lower bound objective, as a function of $\w$, is strictly concave. It suffices to show this for the \textit{active} mode problem, as the problems for the other modes are further constrained maximisation problems of the same objective.

\begin{lemma}
  Fix a bandit instance $\vmu \in \mathcal{L}$. Let $\lambda \mapsto d(\mu_{a,j}, \lambda)$ be a strongly convex function for each arm $a$ and subpopulation $j$. Then for the ABC-S problem with $\vbeta$ such that $\beta_j \neq 0$ for all $j$, the oracle weights $\w^*(\vmu)$  are unique.
\end{lemma}

\begin{proof}
\newcommand{\q}{\bm{q}}
Let $\w^*(\vmu)$ be any oracle weights at $\vmu$. We will show the lower bound objective \eqref{eq:seerank} is strictly concave as a function of $\w$ around $\w^*(\vmu)$, so that $\w^*(\vmu)$ was in fact unique. For each $k>0$, let $\vlambda^k$ be the minimiser in $\Alt^k(\vmu)$ of the weighted divergence in \eqref{eq:seerank}.

We perform a second-order Taylor expansion of the inner objective around $\vlambda^k$, which is a good approximation near $\vlambda^k$ (which is, after all, what matters when reasoning about $\w$ near $\w^*(\vmu)$). To this end, let us abbreviate the divergences, and their first and second derivatives in their second argument by $d_{aj}^k \df d(\mu_{a,j}, \lambda^k_{a,j})$, $g_{aj}^k \df d'(\mu_{a,j}, \lambda^k_{a,j})$ and $h_{aj}^k \df d''(\mu_{a,j}, \lambda^k_{a,j})$, which all depend on $\vlambda^k$. A second-order Taylor expansion of the inner objective of \eqref{eq:seerank} around $\vlambda^k$ yields
\[
    \inf_{\vlambda \in \Alt_k}~ \sum_{a,j} w_{a,j} d(\mu_{a,j}, \lambda_{a,j})
    ~\approx~
  \sum_{a \in \set{0,k}, j} w_{a,j} \del*{
    d_{aj}^k
    -  \frac{(g_{aj}^k)^2}{2 h_{aj}^k}
  }
  + \frac{
    \del*{
      \sum_j \beta_j \del*{
        \frac{
          g^k_{0j}
        }{
          h^k_{0j}
        }
        -
        \frac{
          g^k_{kj}
        }{
          h^k_{kj}
        }
      }
    }^2
  }{
    2 \sum_{a \in \set{0,k}, j} \frac{\beta_j^2}{w_{a,j} h_{aj}^k}
  }
\]
where the optimiser is given by
\[
  \lambda_{a,j}
  ~=~
  \lambda_{a,j}^k
  - \frac{g_{aj}^k}{h_{aj}^k}
  + \frac{\beta_j (\delta_{a=0}-\delta_{a=k})}{ w_{a,j} h_{aj}^k}
  \frac{
    \sum_j \beta_j \del*{
        \frac{
          g^k_{0j}
        }{
          h^k_{0j}
        }
        -
        \frac{
          g^k_{kj}
        }{
          h^k_{kj}
        }
      }
  }{
    \sum_{a \in \{0,k\}, j} \frac{\beta_j^2}{w_{a,j} h_{aj}^k}
  } \;.
\]
Due to the last term, each of these is a strictly concave function of $w_{a,j}$ for $a \in \set{0,k}$ and all $j \le J$ (here we use $\beta_j \neq 0$ and strong convexity $h_{aj}^k > 0$).

Now we still need to consider the $\max_{\w \in \Sigma_{(K+1) \times J}} \min_{k>0}$ problem. Let's convexify this for the inside finite min, and min-max swap to get a problem of the form $\min_{\q \in \Sigma_K} \max_{\w \in \Sigma_{(K+1) \times J}}$. Fixing the minimax outer strategy for $\q$, we find that $\w$ is the maximiser of the strictly concave function
\[
  \w ~\mapsto~ \sum_{k>0} q_k \del*{
      \sum_{a \in \set{0,k}, j} w_{a,j} \del*{
    d_{aj}^k
    -  \frac{(g_{aj}^k)^2}{2 h_{aj}^k}
  }
  + \frac{
    \del*{
      \sum_j \beta_j \del*{
        \frac{
          g^k_{0j}
        }{
          h^k_{0j}
        }
        -
        \frac{
          g^k_{kj}
        }{
          h^k_{kj}
        }
      }
    }^2
  }{
    2 \sum_{a\in \set{0,k}, j} \frac{\beta_j^2}{w_{a,j} h_{aj}^k}
  }
}
\]
To complete the argument, we argue that $q_k > 0$ for all $k > 0$, or, equivalently, that at $\w^*$ the $\min_{k>0}$ are all equalised. For if not, we can move mass from $w_{k,j}$ for the higher $k>0$ to $w_{k',j}$ for the lower $k'$ and increase the objective value. This then proves that $\w^*(\vmu)$ is unique, as the objective function is bounded above by a strictly concave function itself maximised at $\w = \w^*(\vmu)$.
\end{proof}

\section{Algorithm Details}\label{sec:algdetails}
In this section we go into more details on the algorithm for each mode. Let us start with some notation. Let $\beta(\delta,t)$ be a threshold function. We denote the inverse of $\beta(t, \delta)$ in its second argument by
\[
  \beta^{-1}(t, \Lambda)
  ~=~
  \min \setc*{\delta \in (0,1)}{
    \Lambda
    \ge \beta(t, \delta)
  }
  .
\]
We extend the definition of the GLR statistic to sample frequencies $\w$ and bandit $\vmu$ by
\[
  \Lambda(\w, \vmu)
  ~:=~
  \min_{b \neq 0}
  \inf_{
    \substack{\vlambda \in \mathcal L:
      \lambda_0 = \lambda_b}
  }
  \sum_{a \in \{0,b\}} \sum_{i=1}^J w_{a,i} d(\mu_{a,i}, \lambda_{a,i}) \;,
\]
so that the original definition \eqref{eq:GLR} is $\Lambda(t) = \Lambda(\bm{N}(t)/t, \hat \vmu(t))$. For any $\vmu$, we denote by $\nabla_\w \Lambda(\w, \vmu)$ any sub-gradient of $\w \mapsto \Lambda(\w, \vmu)$. We can obtain one such a sub-gradient by letting $(b, \vlambda)$ be any minimiser of $\Lambda(\w, \vmu)$, and constructing the vector with entry $(a,i)$ given by
\[
  (a,i) ~\mapsto~
  \begin{cases}
    d(\mu_{a,i}, \lambda_{a,i})
    & \text{if $a \in \set{0,b}$}
    \\
    0 & \text{otherwise}
  \end{cases}
 \; .
\]
Our algorithms will make use of an online learning method (called $\mathcal A$ below) for linear losses defined on the simplex. This online learning task is known as the Hedge or Experts setting in the literature. We will make use of AdaHedge \cite{ftl.jmlr}, as it adapts automatically to the range of the losses and does not require tuning. Our methods for the active, proportional and agnostic modes are displayed as Algorithms~\ref{alg:active},~\ref{alg:proportional} and~\ref{alg:agnostic}. Each algorithm consists of a Forced Exploration part, which serves to ensure that the empirical estimate of the bandit model converges, i.e.\ $\hat \vmu(t) \to \vmu$. By forcing exploration sub-linearly often, the main term in the sample complexity is unaffected asymptotically. Each algorithm further makes use of online learning to compute $\w^*(\vmu)$. In the notation of this section, we have
\[
  \w^*(\vmu)
  ~=~
  \operatorname*{argmax}_{\w \in \mathcal{C}}~
  \Lambda(\w, \vmu) \;.
\]
Our approach to learning $\w^*(\vmu)$ is to perform gradient steps on the plug-in loss function $\w \mapsto -\Lambda(\w, \hat \vmu(t))$. It is in the convex domain $\mathcal C \subseteq \Sigma_{(K+1)\times J}$ that we see the main difference between the three modes. Recall from Section~\ref{sec:mode.as.constraints} that in the active mode $\w$ is not constrained further, in the proportional mode the subpopulation marginal of $\w$ must equal $\valpha$, i.e.\ $\tuple{\bm{1}, \w} = \valpha$, and in the agnostic mode $\w$ must be the independent product $\w = \bm{v} \valpha$ of some arm marginal $\bm{v} \in \Sigma_{K+1}$ and the subpopulation frequencies $\valpha$. We hence need to design online learners for each of the three $\mathcal C$. In the active case, we have one learner $\mathcal A$ that learns the full joint $\w^*(a,j)$ directly, in the proportional case we use one learner $\mathcal A_j$ for each subpopulation $j \in [J]$ to learn the conditional distribution $\w^*(a|j)$, and in the agnostic case we again use one learner to learn the common marginal $\w^*(a)$. This difference is reflected in the loss function used in each mode, and hence in the gradient that is fed to each learner. In the active case we use the full $(K+1)\times J$ gradients
\[
  \bm{\ell}_t^\text{active}
  ~:=~
  -\nabla_\w \Lambda(\w_t, \hat \vmu(t)) \;.
\]
In the proportional case we have $\w(a,i) = \w(a|i) \alpha_i$, and by the chain rule we hence have gradients
\[
  \bm{\ell}_t^\text{i, proportional}
  ~:=~
  -\nabla_{\w(a|i)} \Lambda(\w_t, \hat \vmu(t))
  ~=~
  - \alpha_i \nabla_{\w} \Lambda(\w_t, \hat \vmu(t)) \bm{e}_i \;.
\]
Finally, in the agnostic case we have $\w(a,i) = w(a) \alpha_i$, and again by the chain rule we have
\[
  \bm{\ell}_t^\text{agnostic}
  ~:=~
  -\nabla_{\w(a)} \Lambda(\w_t, \hat \vmu(t))
  ~=~
  - \nabla_{\w} \Lambda(\w_t, \hat \vmu(t)) \valpha \;.
\]

\paragraph{Run Time}
In each of the three modes, our algorithms evaluate $\Lambda$ for the confidence in the recommendation, compute one sub-gradient of $\Lambda$ for the loss function, and spend $O(K \times J)$ time bookkeeping. Evaluation and sub-gradient computation for $\Lambda$ boil down to solving a convex minimisation problem with an equality constraint. We use Newton's method with backtracking line search to find the minimiser given $b$. Each Newton iteration takes $O(J^2)$ time (recall that only $2$ arms are involved), and we never needed more than $40$. Doing this $K$ times for the explicit minimum over $b$ yields a total per iteration run time of $O(K J^2)$.

\begin{algorithm}[p]
  \begin{algorithmic}
    \Require Online learner $\mathcal A$ for $(K+1)\times J$ experts.
    \For{$t=1, 2, \ldots$}
    \If{any pair $(a,i)$ has $N_{a,i}(t-1) \le \sqrt{t}$}
    \State Pick $A_t, I_t$ any such pair \Comment{Forced Exploration}
    \State Obtain sample $X_t$ from $\nu_{A_t,I_t}$.
    \Else
    \State Get $\w_t$ from online learner $\mathcal A$
    \State Pick $(A_t, I_t) \in \argmin_{a,i} N_{a,i}(t-1) - t w_t(a,i)$ \Comment{Direct Tracking}
    \State Obtain sample $X_t$ from $\nu_{A_t,I_t}$.
    \State Send loss vector $\bm{\ell}_t = -\nabla_{\w} \Lambda(\w_t, \hat \vmu(t))$ to $\mathcal A$
    \EndIf
    \State Recommend $\hat{\mathcal{S}}_{t} = \mathcal{S}_\vbeta(\hat \vmu(t))$ at confidence $\delta_t = \beta^{-1}(t, \Lambda(\bm{N}(t)/t, \hat \vmu(t)))$.
    \EndFor
  \end{algorithmic}
  \caption{Algorithm for Active Mode. \label{alg:active}}
\end{algorithm}

\begin{algorithm}[p]
  \begin{algorithmic}
    \Require $J$ online learners $\mathcal A^{(1)}, \ldots, \mathcal A^{(J)}$ for $(K+1)$ experts each.
    \For{$t=1, 2, \ldots$}
    \State See $I_t \sim \valpha$.
    \If{any arm $a$ has $N_{a,I_t}(t-1) \le \sqrt{\sum_a N_{a,I_t}(t-1)}$}
    \State Pick $A_t$ any such arm \Comment{Forced Exploration}
    \State Obtain sample $X_t$ from $\nu_{A_t,I_t}$.
    \Else
    \State Get $\w_t^{(j)}$ from each online learner $\mathcal A^{(j)}$
    \State Pick $A_t \in \argmin_{a} N_{a,I_t}(t-1) - t w_t^{(I_t)}(a)$ \Comment{Direct Tracking}
    \State Obtain sample $X_t$ from $\nu_{A_t,I_t}$.
    \State For $j \in [J]$, send loss vector $\bm{\ell}_t^{(j)} = - \alpha_j \nabla_{\w} \Lambda\del*{[\alpha_1 \w_t^{(1)} \cdots \alpha_J \w_t^{(J)}], \hat \vmu(t)} \bm{e}_j$ to $\mathcal A^{(j)}$
    \EndIf
    \State Recommend $\hat{\mathcal{S}}_{t} = \mathcal{S}_\vbeta(\hat \vmu(t))$ at confidence $\delta_t = \beta^{-1}(t, \Lambda(\bm{N}(t)/t, \hat \vmu(t)))$.
    \EndFor
  \end{algorithmic}
  \caption{Algorithm for Proportional Mode. \label{alg:proportional}}
\end{algorithm}

\begin{algorithm}[p]
  \begin{algorithmic}
    \Require Online learner $\mathcal A$ for $(K+1)$ experts.
    \For{$t=1, 2, \ldots$}
    \If{any arm $a$ has $\sum_{j=1}^J N_{a,j}(t-1) \le \sqrt{t}$}
    \State Pick $A_t$ any such arm \Comment{Forced Exploration}
    \State See $I_t \sim \valpha$.
    \State Obtain sample $X_t$ from $\nu_{A_t,I_t}$.
    \Else
    \State Get $\w_t$ from online learner $\mathcal A$
    \State Pick $A_t \in \argmin_{a} \sum_{j=1}^J N_{a,j}(t-1) - t w_t(a)$ \Comment{Direct Tracking}
    \State Obtain sample $X_t$ from $\nu_{A_t,I_t}$.
    \State See $I_t \sim \valpha$.
    \State Send loss vector $\bm{\ell}_t = - \nabla_{\w} \Lambda\del*{\w_t \valpha^\top, \hat \vmu(t)} \valpha$ to  $\mathcal A$.
    \EndIf
    \State Recommend $\hat{\mathcal{S}}_{t} = \mathcal{S}_\vbeta(\hat \vmu(t))$ at confidence $\delta_t = \beta^{-1}(t, \Lambda(\bm{N}(t)/t, \hat \vmu(t)))$.
    \EndFor
  \end{algorithmic}
  \caption{Algorithm for Agnostic Mode. \label{alg:agnostic}}
\end{algorithm}

\clearpage

\section{Details of A/B/n experiment}

\label{appx:booking.data.lbs}

\begin{figure}[h]
\begin{subfigure}[T]{0.49\textwidth}
\centering
\begin{tabular}{l|llll|}
& 1 & 2 & 3 &4 \\
\hline
0&0.0296 & 0.0372 & 0.0588 & 0.0620\\
1&0.0300 & 0.0373 & 0.0596 & 0.0626\\
2&0.0295 & 0.0373 & 0.0591 & 0.0630\\
\\
\end{tabular}
\caption{Estimated click probabilities for the different options and seasons}
\end{subfigure}
~
\begin{subfigure}[T]{0.49\textwidth}
\centering
\begin{tabular}{l|l|}
 &1 \\
\hline
0&0.1958\\
1&0.2950\\
2&0.2813\\
3&0.2279
\end{tabular}
\caption{frequency and importance vector ($\valpha=\vbeta$)}
\end{subfigure}

\begin{subfigure}[T]{0.49\textwidth}
\centering
\begin{tabular}{l|llll|}
& 1 & 2 & 3 &4 \\
\hline
0&0.0719 & 0.1222 & 0.1311 & 0.1238\\
1&0.0215 & 0.0475 & 0.0340 & 0.0352\\
2&0.0614 & 0.1060 & 0.1377 & 0.1078\\
\end{tabular}
\caption{$\w^*$ for \textit{active}, $T^* =  3.98 \cdot 10^6 $}
\end{subfigure}
~
\begin{subfigure}[T]{0.49\textwidth}
\centering
\begin{tabular}{l|llll|}
& 1 & 2 & 3 &4 \\
\hline
0&0.0740 & 0.1246 & 0.1476 & 0.1226\\
1&0.0108 & 0.0179 & 0.0214 & 0.0178\\
2&0.0727 & 0.1228 & 0.1460 & 0.1218\\
\end{tabular}
\caption{Sampling proportions \textit{active}}
\end{subfigure}

\begin{subfigure}[T]{0.49\textwidth}
\centering
\begin{tabular}{l|llll|}
& 1 & 2 & 3 &4 \\
\hline
0&0.0814 & 0.1269 & 0.1289 & 0.0889\\
1&0.0230 & 0.0355 & 0.0457 & 0.0277\\
2&0.0914 & 0.1326 & 0.1067 & 0.1113\\
\end{tabular}
\caption{$\w^*$ for \textit{proportional}, $T^* = 4.06 \cdot 10^6 $}
\end{subfigure}
~
\begin{subfigure}[T]{0.49\textwidth}
\centering
\begin{tabular}{l|llll|}
& 1 & 2 & 3 &4 \\
\hline
0&0.0912 & 0.1374 & 0.1307 & 0.1056\\
1&0.0148 & 0.0223 & 0.0214 & 0.0173\\
2&0.0898 & 0.1353 & 0.1293 & 0.1048\\
\end{tabular}
\caption{Sampling proportions \textit{proportional}}
\end{subfigure}

\begin{subfigure}[T]{0.49\textwidth}
\centering
\begin{tabular}{l|l|}
 &1 \\
\hline
0&0.44482\\
1&0.11111\\
2&0.44406\\
\end{tabular}
\caption{$\w^*$ for \textit{agnostic}, $T^* = 4.61 \cdot 10^6 $}
\end{subfigure}
~~
\begin{subfigure}[T]{0.49\textwidth}
\centering
\begin{tabular}{l|l|}
 &1 \\
\hline
0&0.4648\\
1&0.0766\\
2&0.4587\\
\end{tabular}
\caption{Sampling proportions \textit{agnostic}}
\end{subfigure}

\begin{subfigure}[T]{0.49\textwidth}
\centering
\begin{tabular}{l|l|}
 &1 \\
\hline
0&0.44480\\
1&0.11111\\
2&0.44409\\
\end{tabular}
\caption{$\w^*$ for \textit{oblivious}, $T^* = 4.63 \cdot 10^6 $}
\end{subfigure}
~
\begin{subfigure}[T]{0.49\textwidth}
\centering
\begin{tabular}{l|l|}
 &1 \\
\hline
0&0.4647\\
1&0.0764\\
2&0.4589\
\end{tabular}
\caption{Sampling proportions \textit{oblivious}}
\end{subfigure}

\begin{subfigure}[T]{0.49\textwidth}
\centering
\begin{tabular}{l|l|}
 &1 \\
\hline
0&0.500\\
1&0.076\\
2&0.424\\
\end{tabular}
\caption{Sampling proportions BC-ABC}
\end{subfigure}

\caption{Summary of oracle weights and sampling proportions for A/B/n experiment}
\end{figure}

\end{document}